\documentclass[11pt]{article}
\pdfoutput=1

\usepackage{authblk}
\usepackage{times}
\usepackage{fullpage,natbib}
\usepackage[protrusion=false,expansion=true]{microtype}
\usepackage{colortbl}  % Use colortbl for coloring tables
\usepackage[dvipsnames]{xcolor}
\usepackage[colorlinks=true,
linkcolor=RoyalBlue,
anchorcolor=RoyalBlue,
citecolor=CadetBlue
]{hyperref}
\usepackage{amsmath}
\usepackage{amssymb}
\usepackage{mathtools}
\usepackage{amsthm}
\usepackage{mathrsfs}
\usepackage{enumitem}
\usepackage{pifont}
\usepackage{bbm}

\usepackage{booktabs,multicol,multirow}
\usepackage{tabularx}
\usepackage{array}

\usepackage[noabbrev]{cleveref}
\crefname{equation}{}{}
\crefname{assumption}{assumption}{assumptions}
\newlist{assumpenum}{enumerate}{1} % also creates a counter called 'assumpenumi'
\setlist[assumpenum]{label=\Alph*., ref=\theassumption\Alph*, leftmargin=*}
\crefalias{assumpenumi}{assumption}
\newtheorem{example}{Example}
\theoremstyle{plain}
\newtheorem{theorem}{Theorem}[section]
\newtheorem{proposition}[theorem]{Proposition}
\newtheorem{lemma}[theorem]{Lemma}

\theoremstyle{definition}
\newtheorem{definition}[theorem]{Definition}
\newtheorem{assumption}[theorem]{Assumption}
\newtheorem{fact}{Fact}
\theoremstyle{remark}

\def\Loss{\mathcal{L}}
\def\loss{\ell}

\def\x{\mathbf{x}}
\def\lr{\eta}

\def\logisticLoss{\loss_{\mathsf{log}}}
\def\expLoss{\loss_{\mathsf{exp}}}
\newcommand{\norm}[2]{\left\|#2\right\|_{#1}}
\def\l{\left}
\def\r{\right}
\def\R{\mathbb{R}}
\def\w{\mathbf{w}}
\def\wStar{\w^*}
\def\z{\mathbf{z}}
\newcommand{\innerProduct}[2]{\l\langle#1,#2\r\rangle}

\def\bx{\x}
\def\bw{\w}
\def\bz{\z}

\def\bigO{\mathcal{O}}
\def\bu{\mathbf{u}}
\def\barw{\overline{\bw}}

\def\lossPrime{\loss^{\prime}}
\def\lossTwoPrime{\loss^{\prime \prime}}
\def\bones{\mathbf{1}}
\def\GeLU{\mathsf{GeLU}}
\def\Softplus{\mathsf{Softplus}}
\def\SiLU{\mathsf{SiLU}}
\def\ReLU{\mathsf{ReLU}}

\def\be{\mathbf{e}}

\def\eps{\varepsilon}

\title{Minimax Optimal Convergence of Gradient Descent in Logistic Regression via Large and Adaptive Stepsizes}

\author[1]{Ruiqi Zhang}
\author[1]{Jingfeng Wu}
\author[1]{Licong Lin}
\author[1,2]{Peter L. Bartlett}

\affil[1]{University of California, Berkeley\\ {\tt \{rqzhang,uuujf,liconglin,peter\}@berkeley.edu}}
\affil[2]{Google DeepMind}

\date{\today}

\begin{document}
\maketitle

\begin{abstract}
We study \emph{gradient descent} (GD) for logistic regression on linearly separable data with stepsizes that adapt to the current risk, scaled by a constant hyperparameter \(\eta\). We show that after 
at most \(1/\gamma^2\) burn-in steps, GD achieves a risk upper bounded by \(\exp(-\Theta(\eta))\), where \(\gamma\) is the margin of the dataset. As \(\eta\) can be arbitrarily large, GD attains an arbitrarily small risk \emph{immediately after the burn-in steps}, though the risk evolution may be \emph{non-monotonic}.

We further construct hard datasets with margin \(\gamma\), where any batch (or online) first-order method requires \(\Omega(1/\gamma^2)\) steps to find a linear separator. Thus, GD with large, adaptive stepsizes is \emph{minimax optimal} among first-order batch methods. Notably, the classical \emph{Perceptron} \citep{novikoff1962convergence}, a first-order online method, also achieves a step complexity of \(1/\gamma^2\), matching GD even in constants.

Finally, our GD analysis extends to a broad class of loss functions and certain two-layer networks.
\end{abstract}

\section{Introduction}
A large class of optimization methods in machine learning are variants of \emph{gradient descent} (GD). In this paper, we study the convergence of GD with \emph{adaptive} stepsizes given by
\begin{equation}\label{eqn.GD.logistic.maintext}\tag{GD}
  \bw_{t+1} := \bw_t - \lr_t \nabla \Loss(\bw_t),\quad t\ge 0,
\end{equation}
where $\Loss(\cdot)$ is the objective to be minimized, $\bw_t \in \R^d$ is the trainable parameters, and $(\lr_t)_{t\ge 0} $ are the stepsizes. 
We allow the stepsize $\lr_t$ to adapt to
the current risk $\Loss(\bw_t)$. 

Classical analyses of GD require the stepsizes to be sufficiently small so that the risk decreases monotonically \citep{nesterov2018lectures}. This is often referred to as the \emph{descent lemma}. 
For example, the descent lemma is satisfied for 
stepsizes such that $\eta_t < 2/ \sup_{\mathcal{I}} \lambda_{\max}(\nabla^2 \Loss(\cdot) )$, where the supremum is taken over the interval $\mathcal{I}$ connecting $\bw_t$ and $\bw_{t+1}$ and $\lambda_{\max}(\cdot)$ is the maximum eigenvalue of a symmetric matrix. 
In this regime, a large volume of theory has been developed to show the convergence of GD in a variety of settings \citep[see][for example]{lan2020first}. 

However, practical deep learning models trained by GD often converge in the long run while suffering from a locally oscillatory risk \citep{wu2018sgd,xing2018walk,lewkowycz2020large,cohen2021gradient}. 
This oscillation occurs when the stepsizes are too large and violate the descent lemma.
This unstable convergence phenomenon is referred to by \citet{cohen2021gradient} as the \textit{edge of stability} (EoS). 
To obtain a reasonable optimization and generalization performance in practice, it is often necessary to use large stepsizes so that GD enters the EoS regime \citep{wu2018sgd,cohen2021gradient}, violating the seemingly theoretically desirable descent lemma.

Surprisingly, a recent line of works showed that GD provably converges faster by violating the descent lemma in various settings \citep[see][for examples; other related works are discussed later in \Cref{sec.related}]{altschuler2024acceleration,wu2024large}.
Specifically, \citet{altschuler2024acceleration} proposed a stepsize scheduler in which GD violates the descent lemma but (occasionally) achieves an improved convergence rate for smooth convex optimization. 
The work by \citet{wu2024large} focused on GD with a constant stepsize for logistic regression with linearly separable data. They showed that GD with a large stepsize can achieve an accelerated convergence rate, while this is impossible if the stepsize is small and enables the descent lemma. 
Note that the stepsizes considered in \citep{altschuler2024acceleration,wu2024large} are \emph{oblivious}, which are determined before the GD run and do not adapt to the evolving risk.

\subsection{Our results}
This work complements the prior theory by considering the convergence of GD with large and \emph{adaptive} stepsizes. We focus on  
logistic regression with linearly separable data. 
Specifically, the risk to be minimized is given by
\begin{equation}\label{eqn.loss.LR.maintext}
    \Loss(\bw) := \frac{1}{n} \sum_{i=1}^n \loss(y_i \bx_i^\top \bw),\quad \ell\in\big\{\expLoss: z\mapsto\exp(-z) ,\; \logisticLoss: z\mapsto\ln(1+\exp(-z)) \big\},
\end{equation}
where the loss function can be exponential or logistic losses and 
the dataset $(\bx_i, y_i)_{i=1}^n$ satisfies the following standard conditions \citep{novikoff1962convergence}:
\begin{assumption}[Linear separability]\label{assumption.data.maintext}
Assume the dataset $\l(\bx_i, y_i\r)_{i=1}^n$ satisfies
\begin{assumpenum}[leftmargin=*]
\item for every $i=1,\dots,n$, $\|\bx_i\| \leq 1$ and $y_i\in\{\pm 1\}$;
\item there is a margin $\gamma > 0$ and a unit vector $\wStar$ such that $y_i \bx_i^\top \wStar \ge \gamma$ for every $i=1,\dots,n$.
\end{assumpenum}
\end{assumption}
Under \Cref{assumption.data.maintext}, the minimizer of the risk in \Cref{eqn.loss.LR.maintext} is at infinity; moreover, the landscape becomes flatter as the risk decreases.
To compensate for this effect, we consider GD with the following \emph{adaptive} stepsizes \citep[proposed by][]{nacson2019convergence,ji2021characterizing},
\[
\eta_t := \eta \cdot (-\ell^{-1})'\circ\Loss(\bw_t) \approx \eta / \Loss(\bw_t),
\]
where $\eta>0$ is a constant hyperparameter.

We make the following significant contributions. 

\paragraph{Benefits of large and adaptive stepsizes.}
We show that GD with adaptive stepsizes achieves improved convergence by entering the EoS regime. 
Specifically, we show that after at most $1/\gamma^2$ burn-in steps, 
GD attains a risk upper bounded by $\exp(-\Theta(\eta))$, which is arbitrarily small by setting $\eta$ large enough.
By doing so, however, the risk evolution could be non-monotonic. 
On the other hand, if the hyperparameter $\eta$ is set such that GD does not enter the EoS regime, we provide examples in which GD needs at least $\Omega(\ln(1/\eps))$ steps to achieve an $\eps$-risk. 
These together justify the benefits of operating in the EoS regime.

In comparison, prior works on the same problem focused on GD with a large but constant stepsize \citep{wu2024large} or adaptive but small stepsizes \citep{ji2021characterizing}.
They established $\widetilde{\bigO}\big(1/(\gamma^2\sqrt{\eps})\big)$ and $\bigO(\ln(1/\eps)/\gamma^2)$ step complexity for GD to attain an $\eps$-risk respectively (see \Cref{tab.comparison,tab.comparison.separator} for comparisons with other related works).
In stark contrast, we show that, using large and adaptive steps together, $1/\gamma^2$ steps are sufficient. 

\paragraph{Minimax lower bounds.} 
Furthermore, we construct hard datasets with margin $\gamma$, in which \emph{any} batch (or online) first-order methods must take $\Omega(1/\gamma^2)$ steps to identify a linear separator of the dataset. Thus, GD with large and adaptive stepsize is minimax optimal (ignoring constant factors) among all first-order batch methods. 
It is worth noting that the seminal \emph{Perceptron} algorithm \citep{novikoff1962convergence}, a first-order online method, also takes $1/\gamma^2$ steps to find a linear separator. This matches GD even with constants. 
% To the best of our knowledge, this is the first proof of the minimax optimality for Perceptron.

\paragraph{General losses and two-layer networks.}
Finally, we extend our results beyond logistic regression. 
We establish conditions on the loss functions that enable our analysis. 
We also obtain similar results for a class of two-layer networks for fitting linearly separable data.

\subsection{Related works}\label{sec.related}
In this part, we review and discuss additional related papers. 

\paragraph{Edge of stability.}
Our research is motivated by the empirically observed \emph{edge of stability} (EoS) phenomenon \citep[see][and references therein]{cohen2021gradient}. Specifically, in practice, GD often induces a convergent yet oscillatory risk, implying the stepsizes are large such that the descent lemma is violated. 
A growing body of papers investigates the theoretical mechanism of EoS under various scenarios \citep{kong2020stochasticity,wang2021large,wang2023good,lyu2022understanding,wang2022analyzing,ma2022beyond,zhu2022understanding,damian2022self,ahn2022understanding,ahn2023learning,chen2023beyond,kreisler2023gradient,even2023s,andriushchenko2023sgd,lu2023benign,chen2024stability}.
We are less interested in characterizing the mechanism of EoS; rather, we focus on understanding the benefits of operating in the EoS regime for improving optimization efficiency. 

\paragraph{Aggresive stepsize schedulers.}
For smooth and (strongly) convex optimization, a recent line of research showed that GD with certain stepsize schedulers converges faster than GD with a constant stepsize \citep{kelner2022big,altschuler2023acceleration,altschuler2024accelerationrandom,altschuler2024acceleration,grimmer2024provably,grimmer2023accelerated,zhang2024accelerated,zhang2024anytime,grimmer2025accelerated}.
Similar to our results, they obtained the acceleration because their stepsize schedulers violate the descent lemma (occasionally).
However, there are several notable differences. 
First, the considered problem classes are not directly comparable. 
They considered smooth and (strongly) convex optimization problems where the minimizers are finite. 
In contrast, we study the problem of finding a linear separator of a linearly separable dataset by minimizing a convex loss; this includes logistic regression as a special case. Although our logistic regression problem is smooth and convex, it does not admit a finite minimizer and, thus, does not belong to their problem class. 
Moreover, their stepsize schedulers are \emph{oblivious}, determined before the GD run, while our stepsizes are \emph{adaptive} to the current risk.

\paragraph{Logistic regression.}
Logistic regression with linearly separable data is a standard testbed for studying the implicit bias of GD \citep[see][and references thereafter]{soudry2018implicit,ji2018risk}. 
Specifically, this means that GD with a small constant stepsize (satisfying the descent lemma) diverges to infinity in norm but converges to the direction that maximizes the margin \citep{soudry2018implicit,ji2018risk}. The same result is later extended to GD with an arbitrarily large constant stepsize by \citet{wu2023implicit}.
Their results imply a risk convergence rate of $\widetilde\bigO(1/t)$ for GD with a small constant stepsize, where $t$ is the number of steps \citep{soudry2018implicit,ji2018risk}. Remarkably, \citet{wu2024large} showed that with a large constant stepsize, GD enters the EoS regime and attains an accelerated $\widetilde \bigO(1/t^2)$ rate.
Our work is directly motivated by \citet{wu2024large}, although we consider GD with large and adaptive stepsizes while they focused on GD with a large constant stepsize. 

Compared to the constant stepsize, it is known that adaptive stepsizes improve the risk (and margin) convergence rate of GD for logistic regression with separable data \citep{nacson2019convergence,ji2021characterizing}. 
Specifically, 
\citet{nacson2019convergence} and \citet{ji2021characterizing} showed risk convergence rates of $\exp(-\Omega(\sqrt{t}) )$ and $\exp(-\Omega(t))$, respectively. But they both considered GD with adaptive but small stepsizes that satisfy the descent lemma. 
In comparison, for the same algorithm, we show that simply letting the stepsizes be large can greatly accelerate the risk convergence. 

Additionally, the risk (and margin) convergence rate can be further accelerated by using momentum techniques with GD \citep{ji2021fast,wang2022accelerated}. 
The work by \citet{ji2021fast} obtained $\exp(-{\Omega}(t^2-\ln(t)))$ risk convergence rate (see \Cref{prop.momentum.GD.maintext}), where the $\ln(t)$ term was later removed by \citet{wang2022accelerated}.
Their results are again limited to small stepsizes such that a certain potential decreases monotonically (see the discussions after \Cref{prop.momentum.GD.maintext}).
In contrast, we show that not using momentum but just using large stepsizes can lead to faster convergence.

Finally, we highlight that our lower bounds suggest that GD with large and adaptive stepsizes is minimax optimal. In contrast, GD with a large but constant stepsize \citep{wu2024large} or adaptive but small stepsizes \citep{ji2021characterizing} or momentum \citep{ji2021fast,wang2022accelerated} are all suboptimal.

\paragraph{Perceptron.}
The seminal paper by 
\citet{novikoff1962convergence} showed that the \emph{Perceptron} algorithm takes $1/\gamma^2$ steps to find a linear separator of a dataset with margin $\gamma$. Our lower bounds suggest this is minimax optimal---this is the first optimality proof for Perceptron to the best of our knowledge. Note that Perceptron is an \emph{online} method while GD for logistic regression considered in this paper is a \emph{batch} method. In our analysis, entering the EoS regime is a key factor enabling GD to be minimax optimal. Interestingly, Perceptron can be viewed as one-pass stochastic gradient descent with stepsize $1$ under the hinge loss, $z\mapsto\max\{0,-z\}$. Since the hinge loss is nonsmooth, stepsize $1$ violates the descent lemma, so Perceptron effectively operates in the EoS regime, too. 
It seems that operating in the EoS regime might be necessary for achieving first-order minimax optimality in this problem.

A recent paper by \citet{tyurin2024logistic} proposed a batch version of the Perceptron algorithm, attaining the same step complexity of $1/\gamma^2$ when translated to our notation. This is also minimax optimal according to our lower bound. Their method can be viewed as GD with adaptive stepsizes. Their stepsize scheme is motivated by Perceptron, while our stepsize scheme is motivated by the loss landscape \citep{nacson2019convergence,ji2021characterizing}. 
We also point out that \citet{tyurin2024logistic} did not provide minimax lower bounds.

A recent paper by \citet{kornowski2024oracle} considered lower bounds for finding linear separators of a linearly separable dataset from a game-theoretic perspective. Our lower bounds are connected to theirs but with several differences.
First, our lower bound for batch methods (\Cref{thm.lower.burn.in.maintext}) is comparable with their Theorem 4.2: their Theorem 4.2 covers more algorithms, but only showing an $\Omega(\gamma^{-2/3})$ lower bound (ignoring dependence on sample size), while our \Cref{thm.lower.burn.in.maintext} shows an $\Omega(\gamma^{-2})$ lower bound but only covers first-order batch methods.
Second, their Theorem 4.1 matches our lower bound in \Cref{thm.lower.bound.one.pass.maintext} for online methods while covering more algorithms. They both can be viewed as solutions to \citep[Section 9.6, Exercise 3]{shalev2014understanding}.
We choose to keep our version as it is easier to use in our context.

\section{Logistic Regression}\label{sec.logistic.regression}
In this section, we present an improved analysis for GD with large and adaptive stepsizes for logistic regression on linearly separable data.

\subsection{Convergence of GD with large and adaptive stepsizes}
Recall the definitions of \Cref{eqn.GD.logistic.maintext} and the objective \Cref{eqn.loss.LR.maintext}. Recall the adaptive stepsizes are defined as 
\begin{equation}\label{eqn.lr.scheduler.maintext}
    \lr_t 
    := \lr \cdot \l(-\loss^{-1}\r)^{\prime}\circ \Loss(\bw_t) 
    =     
    \begin{dcases}
        \frac{\lr}{\Loss(\bw_t)} &  \loss = \expLoss, \\
        \frac{\lr \exp(\Loss(\bw_t))}{\exp\l(\Loss(\bw_t)\r) - 1} &  \loss = \logisticLoss.
    \end{dcases}
\end{equation}
It is easy to check that $\|\nabla^2 \Loss(\bw)\| \le \Loss(\bw)$ under \Cref{assumption.data.maintext}. 
Thus, the landscape becomes flatter as the risk $\Loss(\bw)$ decreases. The adaptive stepsizes \Cref{eqn.lr.scheduler.maintext} are designed to compensate for the flattened curvature \citep{nacson2019convergence,ji2021characterizing}. 
Moreover, we point out that GD with adaptive stepsizes \Cref{eqn.lr.scheduler.maintext} for logistic regression \Cref{eqn.loss.LR.maintext} is equivalent to \citep{ji2021characterizing}
\begin{equation}\label{eqn.transformed.loss.maintext}
   \bw_{t+1} := \bw_t - \eta \nabla \phi(\bw_t),\quad  \phi(\bw) := - \loss^{-1} \l(\Loss(\bw)\r).
\end{equation}
This is GD with a constant stepsize $\eta$ under a transformed objective $\phi(\cdot)$. The objective $\phi(\cdot)$ measures a smoothed but unnormalized margin \citep{hardy1952inequalities}.
This has been exploited by \citet{ji2021characterizing}, where they established a primal-dual analysis of GD and obtained an improved margin convergence rate. 
Unlike \citet{ji2021characterizing}, we focus on the risk convergence and obtain the following improved results.

\begin{theorem}[GD with large and adaptive stepsizes]\label{thm.LR.maintext}
Consider \eqref{eqn.GD.logistic.maintext} with adaptive stepsizes \eqref{eqn.lr.scheduler.maintext} for logistic regression \Cref{eqn.loss.LR.maintext} under \Cref{assumption.data.maintext}. 
Assume without loss of generality that $\bw_0=\mathbf{0}$.
Then for every $t\ge 1$ and $\lr > 0$, we have 
\begin{equation*}
    \Loss\l(\barw_t\r)  \leq 
    \exp\bigg(- \frac{(\gamma^2  (t+1))^2-1}{4\gamma^2 (t+1)} \lr \bigg),\quad 
    \text{where} \ \  \barw_t := \frac{1}{t+1} \sum_{k=0}^{t} \bw_k.
\end{equation*}
In particular, after $1/\gamma^2$ burn-in steps, 
for every 
$\lr > 0$, we have
\begin{equation*}
    \Loss(\barw_t)
    \leq \exp\l( - \frac{\gamma^2 \lr}{4} \r)= \exp\l(-\Theta\l(\lr\r)\r),\quad t\ge \frac{1}{\gamma^2}.
    % = \exp(-\Theta(\eta)).
    % \leq \exp\l( - \frac{\eps \gamma^2 \lr t}{2(1+\eps)^2} \r),
\end{equation*}
\end{theorem}

The proof of \Cref{thm.LR.maintext} is deferred to \Cref{sec.proof.sketch.maintext}.
\Cref{thm.LR.maintext} provides a sharp risk convergence bound for GD with large and adaptive stepsizes. 
Note that the hyperparameter $\eta$ can be chosen arbitrarily large. 
Therefore, with a sufficiently large $\eta$, GD achieves an arbitrarily small risk right after $1/\gamma^2$ burn-in steps. In this case, however, the evolution of the risk might not be monotonic.

In other words, to attain an $\eps$-risk, GD with adaptive and large stepsize only needs $1/\gamma^2$ steps, where the step complexity is independent of targetted risk $\eps$ (but the smallest base stepsize depends on $\eps$).
This is in stark contrast to the step complexity of GD with a large but constant stepsize \citep{wu2024large} or adaptive but small stepsizes \citep{ji2021characterizing} (see \Cref{tab.comparison} and a detailed discussion later in \Cref{sec.comparision}). Moreover, we will show that this $1/\gamma^2$ step complexity is minimax optimal up to constant factors in \Cref{sec.lower.burn.in}.

\paragraph{Benefits of EoS.}
In \Cref{thm.LR.maintext}, GD achieves the $1/\gamma^2$ step complexity by using large stepsizes and entering the EoS regime. 
Our next theorem suggests this is necessary by providing a lower bound on the convergence rate for adaptive stepsize GD that avoids the EoS phase.

\begin{theorem}[A lower bound for GD in the stable regime]\label{thm.lower.bound.stable.phase.maintext}
Consider \eqref{eqn.GD.logistic.maintext} with adaptive stepsizes \eqref{eqn.lr.scheduler.maintext} for logistic regression \Cref{eqn.loss.LR.maintext} with the following dataset
    \begin{equation*}
        \bx_1 = (\gamma,\, \sqrt{1-\gamma^2}),  \quad 
        \bx_2 = (\gamma,\, -\sqrt{1-\gamma^2}),\quad y_1=y_2=1,
    \end{equation*}
    where $0 < \gamma < 0.1.$
This dataset satisfies Assumption \ref{assumption.data.maintext}.
Let $\bw_0 = \mathbf{0}$.
For all hyperparameter $\eta$ such that $(\Loss(\bw_t))_{t\ge 0}$ is nonincreasing, we have
    \begin{equation*}
        \Loss(\barw_t),\, \Loss(\bw_t) \geq %\loss\l(c t\r) \simeq 
        % \loss\l(- c t\r) \simeq 
        \exp\l(- c t\r),\quad t\ge 1,
    \end{equation*}
    where 
    $c> 0$ is a parameter that depends on $\gamma$ but is independent of $t$ and $\lr$. 
\end{theorem}

The proof of \Cref{thm.lower.bound.stable.phase.maintext} is deferred to \Cref{appendix.proof.lower.bound.stable.phase}. \Cref{thm.lower.bound.stable.phase.maintext} is motivated by Theorem 3 in \citep{wu2024large}, which provides a risk lower bound for GD with a constant stepsize that satisfies the descent lemma.
\Cref{thm.lower.bound.stable.phase.maintext} shows that if GD with adaptive stepsizes satisfies the descent lemma, then the step complexity is at least $\Omega(\ln(1/\eps))$ in the worst case. This contrasts to \Cref{thm.LR.maintext} where GD archives an $\eps$-independent step complexity by violating the descent lemma. 
\Cref{thm.LR.maintext,thm.lower.bound.stable.phase.maintext} together justify the optimization benefits of adaptive stepsize GD to operate in the EoS regime.

\subsection{Comparisons with Prior Results}\label{sec.comparision}

In this part, we review representative existing results on variants of GD for logistic regression with linearly separable data \citep{ji2018risk,ji2021characterizing,ji2021fast,wu2024large} and compare their results with ours. 
\Cref{tab.comparison} provides an overview of the comparisons. We discuss each result in detail below.

\renewcommand{\arraystretch}{1.5}
\begin{table}
    \centering
    \caption{Step complexities for GD with various designs to achieve an $\eps$-risk for logistic regression. 
    } 
    \label{tab.comparison}
    \begin{tabular}{|>{\centering\arraybackslash}m{10cm}|
                    >{\centering\arraybackslash}m{5.5cm}|}
    \hline
    stepsize/momentum design & step complexity  \\
    \hline
    small, constant \citep[Theorem 3.1]{ji2018risk} &  $\widetilde{\bigO}\l(1/\l(\gamma^2\eps\r)\r)$ \\
    \hline
    large, constant \citep[Corollary 2]{wu2024large} &  $\widetilde{\bigO}\l(1/\l(\gamma^2\sqrt{\eps}\r)\r)$ for $\eps <\Theta( 1/n )$ \\
    \hline
    small, adaptive \citep[Theorem 2.2]{ji2021characterizing} & {$\bigO\l(\ln\l(1/\eps\r)/\gamma^2\r)$ } \\
    \hline
    small, adaptive, and momentum \citep[Theorem 3.1]{ji2021fast} & 
    {$\bigO\big(\sqrt{\ln(1/\eps) + \ln(n) \ln \ln(n)}/\gamma\big)$} \\
    \hline
    \rowcolor{red!10} \textbf{large, adaptive (\Cref{thm.LR.maintext})} & {$\le 1/\gamma^2$} \\
    \hline
    \rowcolor{blue!10} \textbf{minimax lower bound (\Cref{thm.lower.burn.in.maintext})} & {$\Omega(1/\gamma^2) $} \\
    \hline
    \end{tabular}
\end{table}
\renewcommand{\arraystretch}{1.0}

\paragraph{A constant stepsize.} 
The work by \citet{ji2018risk} considered GD with a small constant stepsize satisfying the descent lemma and obtained a $\widetilde{\bigO}(1/(\gamma^2 t))$ convergence rate (see their Theorem 3.1). This translates to a $\widetilde{\bigO}(1/(\gamma^2 \eps))$ step complexity.
Later, the work by \citet{wu2024large} obtained a faster rate by considering GD with a large constant stepsize that violates the descent lemma. 
Specifically, they showed the following.

\begin{proposition}[Corollary 2 in \citep{wu2024large}]\label{prop.constant.stepsize.maintext}
    Consider \eqref{eqn.GD.logistic.maintext} with constant stepsize $\lr_t = \lr > 0$ for logistic regression \Cref{eqn.loss.LR.maintext} with logistic loss $\logisticLoss$ under \Cref{assumption.data.maintext}.
    Let $\bw_0=0$.
    For a given step budget $T \geq \max\l\{e,n\r\} / \gamma^2$, there exists $\lr = \Theta\l(T\r)$ such that
    \begin{equation*}
        \Loss\l(\bw_T\r) \leq C  \frac{\ln^2 \l(T\r)}{\gamma^4 T^2},
    \end{equation*}
    where $C > 1$ is a numerical constant.
\end{proposition}
\Cref{prop.constant.stepsize.maintext} leads to an improved step complexity, $\widetilde{\bigO}(1/(\gamma^2 \sqrt{\eps}))$, for GD with a large constant stepsize. 
There are three caveats. First, this improvement only happens for $\eps < \Theta(1/n)$ (hence, it does not help with finding a linear separator; see \Cref{sec.lower.burn.in}). Second, this improvement works under the logistic loss but not under the exponential loss; in fact, \citet{wu2023implicit} constructed a separable dataset where GD with a large constant stepsize under the exponential loss does not converge (see their Theorem 4.2). This gap stems from the logistic loss being Lipschitz while the exponential loss is not. 
Finally, the stepsize is a function of the step budget, meaning that the algorithm needs to be rerun from the beginning when the step budget is changed. 

Compared to their results for GD with a constant stepsize \citep{ji2018risk,wu2024large}, we show that GD with large and adaptive stepsizes \Cref{eqn.lr.scheduler.maintext} achieves a strictly better step complexity of $1/\gamma^2$.
Interestingly, our analysis allows for both logistic and exponential losses. 
This is not contradictory to the counter-example offered by \citet{wu2023implicit}.
Recall that GD with adaptive stepsizes \Cref{eqn.lr.scheduler.maintext} can be viewed as GD with a constant stepsize under a transformed objective $\phi(\cdot)$ defined in \Cref{eqn.transformed.loss.maintext}. 
While the exponential loss is not Lipschitz, the transformed objective $\phi(\cdot)$ is Lipschitz (see~\Cref{lemma.1.lip} in~\Cref{appendix.logistic.result}), enabling large stepsizes.
Finally, unlike \Cref{prop.constant.stepsize.maintext}, the adaptive stepsize scheduler \Cref{eqn.lr.scheduler.maintext} used in \Cref{thm.LR.maintext} is independent of the step budget (as long as setting $\eta$ sufficiently large).

\paragraph{Small adaptive stepsizes.}
Before our paper, the best analysis for GD with adaptive stepsizes \Cref{eqn.lr.scheduler.maintext} is by \citet{ji2021characterizing}.
Their results only allow small stepsizes, summarized as follows.

\begin{proposition}[Consequences of Theorem 2.2 in \citep{ji2021characterizing}]\label{prop.small.stepsize.maintext}
Consider \eqref{eqn.GD.logistic.maintext} with adaptive stepsizes \eqref{eqn.lr.scheduler.maintext} for logistic regression \Cref{eqn.loss.LR.maintext} with exponential loss $\expLoss$ under \Cref{assumption.data.maintext}.
Then the transformed loss $\phi$ (defined in~\eqref{eqn.transformed.loss.maintext}) is $1$-smooth with respect to $\loss_{\infty}$-norm.
Let $\bw_0=0$.
Then for every $\lr \leq 1$, we have $ \Loss\l(\bw_t\r)$ decreases monotonically and
\begin{equation*}
    \Loss\l(\bw_t\r) \leq C\exp\l( - {\gamma^2}\eta t\r),
\end{equation*}
where $C>1$ is a numerical constant.
\end{proposition} 

We note that the main focus of \citet{ji2021characterizing} is to prove a fast margin convergence rate, and \Cref{prop.small.stepsize.maintext} is merely a side product of their results.
\Cref{prop.small.stepsize.maintext} leads to an $\bigO(\ln(1/\eps)/\gamma^2)$ step complexity for GD with small adaptive stepsizes.
In comparison, we analyze the same algorithm, but our \Cref{thm.general.loss.maintext} applies to both large and small adaptive stepsizes.
Our results suggest that the step complexity reduces to $1/\gamma^2$ when the stepsizes are sufficiently large.

\paragraph{Momentum.} 
The work by \citet{ji2021fast} considered a version of momentum GD for logistic regression, achieving a faster margin convergence rate compared to GD. 
Their results also imply a risk convergence rate, summarized as follows.

\begin{proposition}[Consequences of Lemmas C.7 and C.12 in \citep{ji2021fast}]\label{prop.momentum.GD.maintext}
Consider logistic regression \Cref{eqn.loss.LR.maintext} 
with exponential loss $\expLoss$
under \Cref{assumption.data.maintext}.
For a version of momentum GD (see Algorithm 1 in \citet{ji2021fast}) with $\bw_0=0$ and suitably chosen stepsizes, we have
\begin{equation*}
    \Loss\l(\bw_t\r) \leq \exp\l( - C \l(\gamma^2 t^2 - \ln(t+1) \ln(n)\r)\r),\quad t\ge 1,
\end{equation*}
where $C>0$ is a numerical constant.
\end{proposition}

\Cref{prop.momentum.GD.maintext} implies an $\bigO\big(\sqrt{\ln(1/\eps) + \ln(n) \ln\ln(n)}/\gamma \big)$ step complexity for momentum GD.
We remark that with a more careful momentum design, \citet{wang2022accelerated} improved the $\ln(t+1) \ln(n)$ term in the above propostion to $\ln(n)$, which leads to a slightly improved step complexity of $\bigO\big(\sqrt{\ln(1/\eps) + \ln(n)}/\gamma \big)$.
As the improvement is small, we choose to mainly compare with the earlier results by \citet{ji2021fast}.

Recall that our \Cref{thm.general.loss.maintext} suggests a $1/\gamma^2$ step complexity for GD with large and adaptive stepsizes. 
In comparison, the step complexity of their momentum GD has a better dependence on $\gamma$ but has a worse dependence on $\eps$ and $n$. 
As shown later in \Cref{sec.lower.burn.in}, our step complexity is minimax optimal if there are no restrictions on the sample size $n$.
Note that the analysis by \citet{ji2021fast} still relies on the monotonic decreasing of a potential \citep[see][Equation (B.9) in the proof of Lemma B.3]{ji2021fast}, which needs the stepsize to be small.
As their momentum GD is designed to minimize a dual objective with a finite minimizer, it remains unclear whether their momentum GD can be used with large stepsizes. We leave this as future work.

\subsection{Proof of Theorem \ref{thm.LR.maintext}}\label{sec.proof.sketch.maintext}
\begin{proof}[Proof of~\Cref{thm.LR.maintext}]
As explained in \Cref{eqn.transformed.loss.maintext}, it is equivalent to considering GD with a constant stepsize under a transformed objective $\phi(\cdot)$. 
We can check that $\phi(\cdot)$ is convex (see Lemma 5.2 in \citep{ji2021characterizing} or \Cref{lem.convex.phi} in \Cref{appendix.logistic.result}) and $1$-Lipschitz (see~\Cref{lemma.1.lip} in~\Cref{appendix.logistic.result}).
We then use the split optimization technique developed by \citet{wu2024large}.
Specifically, for a comparator $\bu := \bu_1 + \bu_2$, we have
\begin{align}
 \norm{}{\bw_{t+1} - \bu}^2 
=~& \norm{}{\bw_{t} - \bu}^2 + 2\lr \innerProduct{\nabla \phi(\bw_t)}{\bu - \bw_t} + \lr^2 \norm{}{\nabla
\phi(\bw_t)}^2 \notag \\
=~& \norm{}{\bw_{t} - \bu}^2 + 2\lr \innerProduct{\nabla \phi(\bw_t)}{\bu_1 - \bw_t} + \lr \l[2 \innerProduct{\nabla \phi(\bw_t)}{\bu_2} + \lr \norm{}{\nabla \phi(\bw_t)}^2 \r] \notag \\
 {\leq} ~& \norm{}{\bw_{t} - \bu}^2 + 2\lr \innerProduct{\nabla \phi(\bw_t)}{\bu_1 - \bw_t} \label{eqn.a} \\
{\leq} ~& \norm{}{\bw_{t} - \bu}^2 + 2\lr \l(\phi(\bu_1) - \phi(\bw_t)\r), \label{eqn.b}
\end{align}
where \Cref{eqn.a} is by the following inequality (see \Cref{lem.key.maintext} in \Cref{appendix.logistic.result})
\begin{equation}\label{eqn.key.equation.proof.maintext}
    2 \innerProduct{\nabla \phi(\bw)}{\bu_2} + \lr \norm{}{\nabla \phi(\bw)}^2 \leq 0
    \ \  \text{for}\ \ \bu_2 := \frac{\lr}{2\gamma}  \wStar,
\end{equation}
and \Cref{eqn.b} is by the convexity of $\phi(\cdot)$.
Rearranging \eqref{eqn.b} and telescoping the sum, we obtain
\begin{equation}\label{eqn.split.optimization.bound.maintext}
\frac{\norm{}{\bw_{t+1} - \bu}^2}{2\lr (t+1)} + \frac{1}{t+1} \sum_{k=0}^{t} \phi(\bw_k)
\leq \phi(\bu_1)
+ \frac{\norm{}{\bu}^2}{2\lr (t+1)}.
\end{equation}
For $\bu_1\propto \bw^*$, we have $\phi(\bu_1) \leq -\gamma \norm{}{\bu_1}$ by \Cref{assumption.data.maintext}.
Further setting
$
\|\bu_1\| = \gamma \lr (t+1)/2,
$
we get
\begin{align}\label{eqn.split.optimiztio.bound.maintext.2}
\frac{1}{t+1} \sum_{k=0}^{t} \phi(\bw_k)
&\leq
-\gamma \norm{}{\bu_1}
+ \frac{\norm{}{\bu_1 + \bu_2}^2}{2\lr (t+1)}
\leq 
- \frac{(\gamma^2  (t+1))^2-1}{4\gamma^2 (t+1)} \lr.
\end{align}
We complete the proof by applying the convexity of $\phi(\cdot)$ and the fact that $\Loss(\cdot) = \ell(-\phi(\cdot))$.
\end{proof}

\section{Minimax Lower Bounds for First-Order Methods}\label{sec.lower.burn.in}
In this section, we consider the task of finding a linear separator for a linearly separable dataset. Specifically, this means to find a parameter $\hat \bw$ such that 
\[
\min_{i\in[n]} y_i \bx_i^\top \hat\bw > 0.
\]
We point out the fact that an optimization method can find a linear separator by solving the logistic regression problem \Cref{eqn.loss.LR.maintext} sufficiently well.
\begin{fact}\label{lem.risk.separator}
In logistic regression \Cref{eqn.loss.LR.maintext}, if $\Loss(\hat\bw) < \ell(0)/n$, then $\min_{i\in[n]} y_i\bx_i^\top \hat\bw > 0$.
\end{fact}
In this section, we establish minimax lower bounds on the step complexity needed by any first-order methods to solve this task. We then compare the performance for GD with variants designs in this task, as summarized in \Cref{tab.comparison.separator} and will be detailed later in this section.
We start with first-order batch methods and then discuss first-order online methods.

\renewcommand{\arraystretch}{1.5}
\begin{table}
    \centering
    \caption{Step complexities for first-order methods to find a linear separator. 
    } 
    \label{tab.comparison.separator}
    \begin{tabular}{|>{\centering\arraybackslash}m{10cm}|
                >{\centering\arraybackslash}m{1.2cm}|
                >{\centering\arraybackslash}m{3.8cm}|}
    \hline
    stepsize/momentum/loss design & type & step complexity \\
    \hline
    constant \citep[Theorem 3.1]{ji2018risk} & batch & $\widetilde{\bigO}\l(n/\gamma^2\r)$ \\
    \hline
    small, adaptive \citep[Theorem 2.2]{ji2021characterizing} & batch & {$\bigO\l(\ln(n)/\gamma^2\r)$} \\
    \hline
    small, adaptive, and momentum \citep[Theorem 3.1]{ji2021fast} & batch & {$\bigO\big(\sqrt{\ln(n) \ln \ln(n)/\gamma^2}\big)$} \\
    \hline
    normalized batch Perceptron \citep[Theorem 5.3]{tyurin2024logistic} & batch & {$\leq 1/\gamma^2$} \\
    \hline
     \rowcolor{red!10} \textbf{large and adaptive (\Cref{thm.LR.maintext})} & batch & {$\le 1/\gamma^2$} \\
    \hline
    \rowcolor{blue!10} \textbf{minimax lower bound (\Cref{thm.lower.burn.in.maintext})} & batch & {$\Omega(\min\{1/\gamma^2,\, \ln (n)\}) $} \\
    \hline
    constant \citep[Theorem 4]{wu2024large} & online &  {${\bigO}\l(1/\gamma^2\r)$} \\
    \hline
    Perceptron \citep{novikoff1962convergence} & online & {$\le 1/\gamma^2$} \\
    \hline
    \rowcolor{blue!10} \textbf{minimax lower bound (\Cref{thm.lower.bound.one.pass.maintext})} & online & {$\Omega\big(\min\{1/\gamma^2,\, n\}\big)$} \\
    \hline
    \end{tabular}
\end{table}
\renewcommand{\arraystretch}{1.0}

\subsection{A lower bound for first-order batch methods}
We formally define first-order batch methods for our problem as follows. 

\begin{definition}[First-order batch methods]\label{def.first-order-batch}
Let $\ell(\cdot)$ be a locally Lipschitz function.
For each $z$, let $\ell'(z)$ be a unique element from the Clarke subdifferential of $\ell(\cdot)$ at $z$ \citep{clarke1990optimization}.
For a given dataset $(\bx_i, y_i)_{i=1}^n$, define the batch gradient as
\begin{equation*}
    \nabla \Loss(\bw) := \frac{1}{n} \sum_{i=1}^n \lossPrime\big(y_i \bx_i^\top \bw\big) y_i \bx_i,
\end{equation*}
We say $\bw_t$ is the output of a first-order batch method in $t$ steps with initialization $\bw_0$ on dataset $(\bx_i, y_i)_{i=1}^n$, if it can be generated by
\begin{equation*}
    \bw_{k} \in \bw_0 + \mathsf{Lin}\l\{\nabla \Loss(\bw_0),  \dots, \nabla \Loss(\bw_{k-1})\r\},\quad k= 1,\dots, t,
\end{equation*}
where ``$\mathsf{Lin}$'' is the linear span of a vector set
and ``$+$'' is the Minkowski addition. 
\end{definition}

\Cref{def.first-order-batch} characterizes a class of first-order methods for finding linear separators. 
Compared to the class of first-order methods for smooth convex optimization \citep[Assumption 2.1.4]{nesterov2018lectures}, our 
\Cref{def.first-order-batch} requires the predictor to be linear---since the goal is to find a linear separator---but does not require the objective function to be smooth or convex.
We establish the following lower bounds.

\begin{theorem}[A lower bound for first-order batch methods]\label{thm.lower.burn.in.maintext}
For every $0 < \gamma < 1/6$, $n>16$, and $\bw_0$,
there exists a dataset $\l(\bx_i,y_i\r)_{i=1}^n$ 
satisfying Assumption~\ref{assumption.data.maintext} 
such that the following holds.
For any $\bw_t$ output 
by a first-order batch method in $t$-steps with initialization $\bw_0$ on this dataset, we have
\begin{equation*}
\min_{i\in[n]} y_i \bx_i^\top  \bw_t> 0
\quad \text{implies that}\quad 
    t \ge \min\bigg\{\frac{\ln n}{8 \ln 2}, \frac{1}{30 \gamma^2}\bigg\}.
\end{equation*}
\end{theorem}

The proof of \Cref{thm.lower.burn.in.maintext} is deferred to \Cref{appendix.proof.lower.bound.batch}.
The proof is based on a dimension argument, motivated by the classical lower bounds for first-order methods in smooth convex optimization \citep[Theorem 2.1.7]{nesterov2018lectures}.

\paragraph{Minimax optimality.}
Due to \Cref{lem.risk.separator} and \Cref{thm.LR.maintext}, GD with large and adaptive stepsizes takes $1/\gamma^2$ steps to find a linear seprator. 
Our lower bound in \Cref{thm.lower.burn.in.maintext} suggests that this is minimax optimal ignoring constant factors, in the sense that there is no restriction on the sample size $n$ (so $n$ is allowed to be exponential in $1/\gamma$ in the worst case).

\paragraph{Prior results.}
Applying \Cref{lem.risk.separator} to the prior convergence results summarized in \Cref{tab.comparison}, we can conclude the step complexities of the variants of GD considered in previous works \citep{ji2018risk,ji2021characterizing,ji2021fast,wu2024large} for finding a linear separator. 
Additionally, the work by \citet{tyurin2024logistic} provided a first-order batch method called normalized batch
Perceptron, which achieves $1/\gamma^2$ step complexity for finding the linear separator.
These results are summarized in \Cref{tab.comparison.separator}.

We make three remarks.
First, since the acceleration effect of a large constant stepsize only appears for $\eps < \Theta(1/n)$ (see \Cref{tab.comparison}), a large constant stepsize considered by \citet{wu2024large} does no help GD to find a linear separator (but also doest not hurt) compared to the results in \citep{ji2018risk}. 
Second, when $\gamma$ is fixed and the sample size $n$ is allowed to be arbitrary, the methods considered in \citep{ji2018risk,ji2021characterizing,ji2021fast} are all suboptimal in the worst case, while GD with large and adaptive stepsizes and the  normalized batch
Perceptron by \citet{tyurin2024logistic} are minimax optimal. 
Finally, note that \citet{ji2021fast} obtained an $\bigO(\sqrt{\ln(n) \ln \ln(n)/\gamma^2})$ step complexity by using momentum techniques (note that this can be improved to $\bigO(\sqrt{\ln(n)/\gamma^2})$ by \citet{wang2022accelerated} as discussed after \Cref{prop.momentum.GD.maintext}).
Their results do not violate the lower bound in \Cref{thm.lower.burn.in.maintext}, 
but suggest that the $1/\gamma^2$ term in our lower bound might be improvable in the regime where $n=\mathrm{poly}(1/\gamma)$. 

It turns out that identifying the correct trade-off between $n$ and $\gamma$ is challenging. 
Besides \Cref{thm.general.loss.maintext}, we give an alternative dataset constrcution that leads to a lower bound of $\Omega(\min\{\gamma^{-2/3}, n\})$ (see~\Cref{thm.lower.burn.in.alternate} in~\Cref{appendix.alternate.lower.bound}). 
We leave it as future work to prove the first-order minimax step complexity that is tight for all choices of $\gamma$ and $n$.

\subsection{A lower bound for first-order online methods}

As a side product, we also establish a step complexity lower bound for 
first-order online method for finding a linear separator.
We formally define a first-order online method as follows.

\begin{definition}[First-order online methods]\label{def-online-first-order}
Let $\ell(\cdot)$ be a locally Lipschitz function.
For each $z$, let $\ell'(z)$ be a unique element from the Clarke subdifferential of $\ell(\cdot)$ at $z$ \citep{clarke1990optimization}.
We say a sequence $(\bw_k)_{k=0}^t$ is generated by a first-order online method with initialization $\bw_0$ on dataset $(\bx_i, y_i)_{i=1}^t$, if it satisfies
\begin{equation*}
    \bw_{k} \in \bw_0 + \mathsf{Lin}\l\{ \lossPrime\big(y_{i} \bx_{i}^\top \bw_{i-1}\big) y_{i} \bx_{i},\  i=1,\dots, k \r\},\quad k= 1,\dots, t,
\end{equation*}
where ``$\mathsf{Lin}$'' is the linear span of a vector set
and ``$+$'' is the Minkowski addition. 
\end{definition}

The following theorem presents our lower bound for the first-order online method. A version of this theorem (that covers more algorithms) also appears in \citep[Theorem 4.1]{kornowski2024oracle}, both of which can be viewed as solutions of \citep[Section 9.6, Exercise 3]{shalev2014understanding}. Our statement of the lower bound is easier to use in our context.

\begin{theorem}[Lower bounds for online first-order methods]\label{thm.lower.bound.one.pass.maintext}
For every $0 < \gamma < 1/2, n \geq 2,$ and $\bw_0$, 
there exists a dataset $\l(\bx_i,y_i\r)_{i=1}^n$ 
satisfying Assumption~\ref{assumption.data.maintext} such that the following holds.
\begin{itemize}[leftmargin=*]
    \item 
For any sqeuence $(\bw_k)_{k= 0}^t$ generated by a first-order online method with initialization $\bw_0$ and dataset $\big( \bx_{\pi(k)}, y_{\pi(k)} \big)_{k=1}^t$, where $\pi(k)\in[n]$ but is otherwise arbitrary, 
we have
\begin{equation*}
\min_{i\in[n]} y_i \bx_i^\top  \bw_t> 0
\quad \text{implies that}\quad 
    t \ge \min \bigg\{\frac{1}{2\gamma^2},\, n\bigg\}.
\end{equation*}
\item 
For any sqeuence $(\bw_k)_{k=0}^n$ generated by a first-order online method with initialization $\bw_0$ and dataset $\l(\bx_i,y_i\r)_{i=1}^n$, we have
\begin{equation*}
\sum_{k=1}^t \mathbbm{1}\{y_{k} \bx_{k}^\top  \bw_{k-1}\le 0\} \ge \min \bigg\{\frac{1}{2\gamma^2},\, t\bigg\}\quad \text{for every}\ \ 1\le t \le n.
\end{equation*}
\end{itemize}
\end{theorem}
We include the proof of \Cref{thm.lower.bound.one.pass.maintext} in~\Cref{appendix.lower.bound.one.pass} for completeness. 
We remark that the dataset construction in \Cref{thm.lower.bound.one.pass.maintext} is different from the one in 
\Cref{thm.lower.burn.in.maintext}.
From an offline perspective, \Cref{thm.lower.bound.one.pass.maintext} suggests that any first-order online methods need to make $\Omega(1/\gamma^2)$ updates to find a linear separator of a given dataset (assuming that $n$ is large). 
From an online perspective, it makes $\Omega(1/\gamma^2)$ mistake steps (or regret in zero-one loss) for a large $t$ in the worst case.

\paragraph{Perceptron.}
Perceptron is a classical method for finding a linear separator of a linearly separable dataset.
It takes the following updates:
\begin{equation*}
\bw_0=\mathbf{0},\quad 
    \bw_{t} := \bw_{t-1} + \mathbbm{1}\{y_{t} \bx_{t}^\top \bw_{t-1} \le 0 \} y_{t} \bx_{t}, \quad
    t \ge 1.
\end{equation*}
Under \Cref{assumption.data.maintext}, a seminal analysis by \citet{novikoff1962convergence} suggests that Perceptron makes at most $1/\gamma^2$ mistake steps when run over a dataset of arbitrary size (that is, the total regret in zero-one loss is at most $1/\gamma^2$ for any $t$).

Note that Perceptron can be viewed as online stochastic gradient descent under the hinge loss 
$\ell_{\mathrm{hinge}}(z):= \max\{0, -z\}$ with stepsize $1$ and subderivative choice $\ell_{\mathrm{hinge}}'(0):=1$.
Thus, \Cref{thm.lower.bound.one.pass.maintext} applies to Perceptron, suggesting that Perceptron is minimax optimal. This is also known from \citep[Theorem 4.1]{kornowski2024oracle} and \citep[Section 9.6, Exercise 3]{shalev2014understanding}.
% To the best of our knowledge, this is the first proof of the minimax optimality for Perceptron.

We also point out that the hinge loss is non-smooth. So, the stepsize $1$ used in Perceptron also violates the descent lemma (although Perceptron is online), similar to our large stepsize GD. We conjecture that violating the descent lemma might be a fundamental property for attaining first-order optimality in this task.

\paragraph{SGD for logistic regression.}
Similarly to Perceptron, 
online stochastic gradient descent with a (large) constant stepsize for logistic regression also makes $\bigO(1/\gamma^2)$ mistake steps \citep[see the proof of Theorem 4 in][for example]{wu2024large}.
Since the logistic loss is close to the hinge loss when zooming away from zero, we should expect SGD with large stepsizes for logistic regression to behave like Perceptron.

\section{Extensions}
This section extends results in \Cref{sec.logistic.regression} to a two-layer network and a generic class of loss functions.

\subsection{Two-Layer Networks}\label{sec.NN}
We consider a two-layer network defined as \citep{brutzkus2018sgd}
\begin{equation}\label{eqn.def.NN.maintext}
    f(\bw;\bx) := \frac{1}{m} \sum_{j=1}^m a_j \sigma \big(\bx^\top \bw^{(j)} \big), \quad \bw := \big( \bw^{(1)}, \bw^{(2)}, \dots, \bw^{(m)} \big),
\end{equation}
where $m$ is the number of neurons, 
$a_j \in \{\pm 1\}$ for $j=1,\dots,m$ are fixed parameters, $\bw$ are the trainable parameters, and $\sigma(z):= \max\{z, \alpha z\}$ for $0<\alpha<1$ is the leaky ReLU activation.
We fix a choice of subderivative $\sigma^{\prime}(0) := 1$ for clarity. But our results extend to any choice of subderivate $\sigma^{\prime}(0) \in[\alpha, 1]$.
The objective is then given by 
\begin{equation}\label{eqn.loss.NN.maintext}
    \Loss(\bw) := \frac{1}{n}\sum_{i=1}^n \loss\l(y_i f\l(\bw;\bx_i\r)\r),
\end{equation}
where $\loss(\cdot)$ is exponential loss or logistic loss and the dataset satisfies \Cref{assumption.data.maintext}. 
Similarly, we consider GD with adaptive stepsizes \eqref{eqn.lr.scheduler.maintext}. 
The next theorem provides an improved convergence analysis for GD with large and adaptive stepsizes. 

\begin{theorem}[A two-layer network]\label{thm.NN.maintext}
Consider \eqref{eqn.GD.logistic.maintext} with adaptive stepsizes \eqref{eqn.lr.scheduler.maintext} for minimizing objective \eqref{eqn.loss.NN.maintext} under~\Cref{assumption.data.maintext}.
Assume without loss of generality that $\bw_0 = \mathbf{0}.$
Then for every $t \geq 1$ and $\lr > 0$, we have
\begin{equation*}
    \min_{k\leq t} \Loss(\bw_k) \leq 
        \exp\l(-\frac{(\alpha \gamma^2 (t+1))^2 - 1}{4\gamma^2 (t+1)} \lr\r).
\end{equation*}
In particular, after $1/(\alpha\gamma^2 )$ burn-in steps, for every $\lr > 0,$ we have
\begin{equation*}
\min_{k\leq t} \Loss(\bw_k) \leq \exp\l(-\frac{\alpha^2 \gamma^2 \lr}{4} \r) = \exp\l(-\Theta\l(\lr\r)\r),\quad t\ge \frac{1}{\alpha\gamma^2}.
\end{equation*}
\end{theorem}

The proof of \Cref{thm.NN.maintext} is deferred to \Cref{appendix.proof.NN}, which combines techniques from the proof of \Cref{thm.LR.maintext} and the analysis of GD with a large constant stepsize for two-layer networks by \citet{cai2024large}. 
\Cref{thm.NN.maintext} suggests GD with large and adaptive stepsizes attains an atbirarly small risk right after $1/(\alpha \gamma^2)$ burn-in steps. 
In comparison, the work by \citet{cai2024large} only obtained an $\widetilde{\mathcal{O}}(1/t^2)$ rate. 
Other discussions for \Cref{thm.LR.maintext} also apply here in a similar manner. 

We remark that the leaky ReLU activation in \Cref{thm.NN.maintext} can be extended to other near-homogeneous activation functions, such as leaky GeLU, leaky Softplus, and leaky SiLU \citep[see][]{cai2024large}. 
This is done in \Cref{appendix.proof.NN}.

\subsection{General loss functions}\label{sec.general.loss}

In this part, we extend our results in \Cref{sec.logistic.regression} from logistic and exponential losses to a generic class of loss functions.

\begin{assumption}[Loss function  conditions]\label{assumptions.general.loss.maintext}
    Let $\loss: \R \to (0,\infty)$ be a positive and continuously differentiable function. Assume that
    \begin{assumpenum}
        \item\label{asp.item.A} The loss function $\ell$ is positive, strictly decreasing, and convex. 
        \item\label{asp.item.B} For $z > 0$, the inverse $\loss^{-1}(z)$ exists, and is differentiable and decreasing.
        % $\l(\loss^{-1}(z)\r)^{\prime} < 0$ for $z > 0.$
        \item\label{asp.item.C} The transformed objective $\phi(\cdot)$ in \eqref{eqn.transformed.loss.maintext} is convex and $C_\loss$-Lipschitz for a constant $C_\loss \geq 1$. 
    \end{assumpenum}
\end{assumption}

Our next theorem shows that \Cref{thm.LR.maintext} applies to loss functions beyond the logistic and exponential losses, as long as they satisfy \Cref{assumptions.general.loss.maintext}.

\begin{theorem}[General loss functions]\label{thm.general.loss.maintext}
    Consider \eqref{eqn.GD.logistic.maintext} with adaptive stepsizes \eqref{eqn.lr.scheduler.maintext} for objective function $\Loss(\bw) := (1/n)\sum_{i=1}^n \loss\l(y_i \bx_i^\top \bw\r)$ under~\Cref{assumption.data.maintext,assumptions.general.loss.maintext}.
    Assume without loss of generality that $\bw_0 = \mathbf{0}.$
    Then for every $t \geq 1$ and $\lr > 0$, we have
    \begin{equation*}
        \Loss\l(\barw_t\r)  \leq 
        \loss\bigg(-\frac{\l(\gamma^2 (t+1)\r)^2 - C_\loss}{4 \gamma^2 (t+1)} \lr \bigg), \quad \text{where} \ \ \barw_t := \frac{1}{t+1} \sum_{k=0}^{t} \bw_k.
    \end{equation*}
    In particular, after $C_\loss / \gamma^2$ burn-in steps, for every $\lr > 0,$ we have
    \begin{equation*}
        \Loss(\barw_t) 
        \leq \loss \l( - \frac{\gamma^2 \lr}{4} \r), \quad 
        t \geq \frac{C_\loss}{\gamma^2}.
    \end{equation*}
\end{theorem}

The proof of \Cref{thm.general.loss.maintext} is deferred to~\Cref{appendix.proof.general.loss}.
We conclude this section by providing several examples of loss functions that satisfy \Cref{assumptions.general.loss.maintext}. 
The proof is also deferred to 
\Cref{appendix.proof.general.loss}.
% \Cref{appendix.example.loss}.

\begin{example}\label{lem.example}
The following loss functions satisfy \Cref{assumptions.general.loss.maintext}.
\begin{enumerate}[leftmargin=*]
    \item The logistic loss $\ell_{\log}: z\mapsto \ln(1+\exp(-z))$ and the exponential loss $\ell_{\exp}: z\mapsto \exp(-z)$ satisfy \Cref{assumptions.general.loss.maintext} with $C_\loss=1.$
    \item The polynomial loss of degree $k>0$ \citep{ji2021characterizing},
    \begin{equation*}
        \loss_{\mathrm{poly}}(z) := 
        \begin{dcases}
             \frac{1}{(1+z)^k} & z \geq 0, \\
        -2kz + \frac{1}{(1-z)^k} & z \leq 0,
        \end{dcases}
    \end{equation*}
    satisfies~\Cref{assumptions.general.loss.maintext} with $C_\loss = n^{1/k}.$
    \item The semi-circle loss \citep{shen2005loss},
    \begin{equation*}
        \loss_{\mathrm{semi}}(z) := \frac{-z + \sqrt{z^2 + 4}}{2},
    \end{equation*}
    satisfies~\Cref{assumptions.general.loss.maintext} with $C_\loss = n+1.$
\end{enumerate}
\end{example}

\section{Conclusion}
We consider logistic regression on linearly separable data with margin $\gamma$. 
We show that GD with large and adaptive stepsize achieves an arbitrarily small risk within $1/\gamma^2$ steps, although the risk evolution might not be monotonic. 
This is impossible if GD with adaptive stepsize induces a monotonically decreasing risk. 
Additionally, we establish an $\Omega(1/\gamma^2)$ minimax lower bound on the number of steps required by any online or batch first-order algorithms to find a linear separator of the input dataset.
This suggests GD with large and adaptive stepsize is minimax optimal.
Finally, our results extend to a broad class of loss functions and certain two-layer networks.

\section*{Acknowledgments}
We thank Yuhang Cai, Hossein Mobahi, and Matus Telgarsky for their helpful comments.
We thank Qiuyu Ren for his significant help in improving the minimax lower bound.
We thank Guy Kornowski and Ohad Shamir for pointing out missing references on the optimality of Perceptron.
We gratefully acknowledge the NSF's support of FODSI through grant DMS-2023505
and of the NSF and the Simons Foundation for the Collaboration on the Theoretical Foundations of Deep Learning through awards DMS-2031883 and \#814639
and of the ONR through MURI award N000142112431.

\bibliography{main}
\bibliographystyle{ims}
\appendix

\section{Missing Parts in the Proof of Theorem \ref{thm.LR.maintext} in Section \ref{sec.proof.sketch.maintext}}\label{appendix.logistic.result}

We first prove the $1$-Lipschitzness of the transformed loss function $\phi(\cdot)$.

\begin{lemma}\label{lemma.1.lip}
    For the exponential loss and logistic loss, the transformed loss $\phi(\cdot)$ defined in \eqref{eqn.transformed.loss.maintext} is $1$-Lipschitz with respect to $\norm{}{\cdot}$.
\end{lemma}
\begin{proof}[Proof of~\Cref{lemma.1.lip}]
    Recall that $\phi(\bw) := - \loss^{-1}(\Loss(\bw))$ and $\norm{}{\bx_i} \leq 1$ for $1 \leq i \leq n$. We have
    \begin{equation*}
        \norm{}{\nabla \phi\l(\bw\r)}
        = \norm{}{\l(-\loss^{-1}\r)^\prime \l(\Loss(\bw)\r) \cdot \frac{1}{n} \sum_{i=1}^n \loss^{\prime}(y_i \bx_i^\top \bw) y_i \bx_i}
        \leq \frac{1}{n} \sum_{i=1}^n \l|\loss^{\prime}(y_i \bx_i^\top \bw)\r| \cdot \l|\l(\loss^{-1}\r)^\prime \l(\Loss(\bw)\r)\r|
        := \star.
    \end{equation*}
    Let $\loss_i = \loss\l(y_i \bx_i^\top \bw\r).$ We rearrange $\star$ as 
    \begin{equation*}
        \star = \frac{\frac{1}{n} \sum_{i=1}^n \l(-\lossPrime\r) \l(\loss^{-1}\l(\loss_i\r)\r)}{\l(-\lossPrime\r) \l(\loss^{-1} \l(\frac{1}{n} \sum_{i=1}^n \loss_i\r)\r)}
        = \frac{\frac{1}{n} \sum_{=1}^n h\l(\loss_i\r)}{h\l(\frac{1}{n} \sum_{i=1}^n \loss_i\r)},
    \end{equation*}
    where $h(\cdot)$ is defined as
    \begin{equation*}
        h(z) := \l(-\lossPrime\r) \l(\loss^{-1}\l(z\r) \r)
        = \begin{dcases}
            z & \loss = \expLoss, \\
            1 - \exp(-z) & \loss = \logisticLoss.
        \end{dcases}
    \end{equation*}
    For both losses, $h(\cdot)$ is concave on $z > 0.$ Therefore, $\star \leq 1.$
\end{proof}

\begin{lemma}
\label{lem.key.maintext}
Let $\phi(\cdot)$ be the transformed loss function defined in \eqref{eqn.transformed.loss.maintext}.
Under \Cref{assumption.data.maintext}, if $\phi(\cdot)$ is $C_\loss$-Lipschitz, we have
\begin{equation*}
    2 \innerProduct{\nabla \phi(\bw)}{\bu_2} + \lr \norm{}{\nabla \phi(\bw)}^2 \leq 0, \quad \bu_2 := ({C_\loss \lr}/(2\gamma))  \wStar.
\end{equation*}
\end{lemma}
\begin{proof}[Proof of~\Cref{lem.key.maintext}] 
Recall that $\phi(\bw) := - \loss^{-1}(\Loss(\bw))$ and $\bu_2 = (\lr/(2\gamma)) \bw^*.$
We also use the fact that $\innerProduct{\bw^*}{y_i \bx_i} \geq \gamma >0$ and $\norm{}{\bx_i} \leq 1$ for $1 \leq i \leq n$ from~\Cref{assumption.data.maintext}. Applying~\Cref{lemma.1.lip}, we have
\begin{align}
    &2 \innerProduct{\nabla \phi(\bw)}{\bu_2} + \lr \norm{}{\nabla \phi (\bw)}^2 \notag \\
    \leq~& \frac{2}{n} \cdot \l(-\loss^{-1}\r)^\prime \l(\Loss(\bw)\r) \cdot \sum_{i=1}^n \loss^{\prime}(y_i \bx_i^\top \bw) \innerProduct{\bu_2}{y_i \bx_i} + C_\loss \lr \norm{}{\l(-\loss^{-1}\r)^\prime \l(\Loss(\bw)\r) \cdot \frac{1}{n} \sum_{i=1}^n \loss^{\prime}(y_i \bx_i^\top \bw) y_i \bx_i} \tag{$\phi(\cdot)$ is $C_\loss$-Lipschitz}\\
    \leq~& -\frac{2\gamma \norm{}{\bu_2}}{n} \cdot \l|\l(\loss^{-1}\r)^\prime \l(\Loss(\bw)\r)\r| \cdot \sum_{i=1}^n \l|\loss^{\prime}(y_i \bx_i^\top \bw)\r| + C_\loss \lr \l(\l|\l(\loss^{-1}\r)^\prime \l(\Loss(\bw)\r)\r| \cdot \frac{1}{n} \sum_{i=1}^n \l|\loss^{\prime}(y_i \bx_i^\top \bw)\r| \r) \notag \\
    =~&  \frac{1}{n} \sum_{i=1}^n \l|\loss^{\prime}(y_i \bx_i^\top \bw)\r| \cdot \l|\l(\loss^{-1}\r)^\prime \l(\Loss(\bw)\r)\r| \cdot \l[-2 \gamma \norm{}{\bu_2} + C_\loss \lr\r]. \notag
\end{align}
Invoking the definition of $\bu_2$ completes the proof.
\end{proof}

The following lemma is a restatement of Lemma 5.2 in \citep{ji2021characterizing}. 

\begin{lemma}[Lemma 5.2 in \citep{ji2021characterizing}]\label{lem.convex.phi}
    Let $\loss(\cdot)$ be twice continuously differentiable, positive, decreasing, and convex. If $\frac{\loss^{\prime}(t)^2}{\loss(t) \cdot \loss^{\prime \prime}(t)}$ is decreasing on $\R,$ then $\psi(\bz) := -\loss^{-1} (\frac{1}{n} \sum_{i=1}^n \loss(z_i))$ is convex.
    Moreover, $\phi(\cdot)$ is also convex since $\phi(\bw)$ is a composition between a linear mapping and $\psi(\cdot).$
\end{lemma}

\section{Proof of Theorem \ref{thm.lower.bound.stable.phase.maintext}}\label{appendix.proof.lower.bound.stable.phase}
Before proving~\Cref{thm.lower.bound.stable.phase.maintext}, let us first delve into a technical lemma.
This lemma shows that for some special datasets, the stepsize must be small if the loss is monotonically decreasing.
\begin{lemma}[A variant of Lemma 14 in \citep{wu2024large}]\label{lemma.upper.bound.stepize}
    Let $\bw_0 = \mathbf{0}, y_i = 1$ for $i = 1,2,...,n,$ and $\bar{\bx}:= (1/n) \cdot \sum_{i=1}^n \bx_i$. Assume there exist constants $r>0$ and $q \in (0,1)$ such that 
    \begin{equation}\label{eqn.lower.bound.lemma.assumption1}
        \frac{\l|\l\{i \in [n]: \bx_i^\top \bar{\bx} < -r\r\}\r|}{n} \geq q.
    \end{equation}
    Assume we perform the gradient descent in \eqref{eqn.GD.logistic.maintext}. If $\Loss(\bw_1) \leq \Loss(\bw_0),$ then
    $
        \lr \leq \loss(0)/(qr).
    $
\end{lemma}

\begin{proof}[Proof of~\Cref{lemma.upper.bound.stepize}]
    Staring from $\bw_0 = \mathbf{0},$ we have $\Loss(\bw_0) = \loss(0)$, $\nabla \Loss(\bw_0) = \lossPrime(0) \cdot \bar{\bx}$, $\nabla \phi(\bw_0) = (-\loss^{-1})^{\prime}(\loss(0)) \cdot \lossPrime(0) \cdot \bar{\bx}.$ For both the exponential loss and logistic loss, one can verify that $- (-\loss^{-1})^{\prime}(\loss(0)) \cdot \lossPrime(0) = 1,$ which implies
    $
        \bw_1 = \lr \bar{\bx}.
    $
    Suppose $\lr > qr/\loss(0).$ Then,
    \begin{align*}
        \Loss(\bw_1) 
        %&= \frac{1}{n} \sum_{i=1}^n \loss(\lr \bx_i^\top \bar{\bx}) 
        &\geq \frac{1}{n} \sum_{i=1}^n \loss(\lr \bx_i^\top \bar{\bx}) \mathbb{I}\l(\bx_i^\top \bar{\bx} < -r\r)
        \geq \loss(-\lr r) \cdot \frac{1}{n} \sum_{i=1}^n \mathbb{I}\l(\bx_i^\top \bar{\bx} < -r\r)
        \geq \loss(-\lr r) \cdot q \tag{From \eqref{eqn.lower.bound.lemma.assumption1}}\\
        &\geq q \cdot \ln\l(1 + \exp(\lr r)\r)
        \geq \lr q r
        \geq \loss(0) = \Loss(\bw_0),
    \end{align*}
    a contradiction to $\Loss(\bw_1) \leq \Loss(\bw_0).$ 
    Thus $\lr \leq \loss(0)/(qr).$
\end{proof}

Now we prove the~\Cref{thm.lower.bound.stable.phase.maintext}. 
\begin{proof}[Proof of~\Cref{thm.lower.bound.stable.phase.maintext}]
    Consider the specific dataset from \Cref{thm.lower.bound.stable.phase.maintext}, which satisfies the conditions of Lemma \ref{lemma.upper.bound.stepize} with $r = 0.1$ and $q = 0.5.$ So~\Cref{lemma.upper.bound.stepize} implies $\lr \leq c_1$, where $c_1$ is a constant that depends on the loss (but \textit{not} on $t$). Next, from the GD iterate on the transformed loss~\Cref{eqn.transformed.loss.maintext} and the 1-Lipschitzness of $\phi(\cdot)$ (\Cref{lemma.1.lip}), we have
    \begin{equation*}
        \norm{}{\bw_t} = \norm{}{\lr \sum_{k=1}^{t-1} \nabla \phi(\bw_k)} \leq \lr \cdot \sum_{k=0}^{t-1} \norm{}{\nabla \phi(\bw_k)}
        \leq \lr t.
    \end{equation*}
    Since $\norm{}{\bx_i} \leq 1$, it follows
    \begin{equation*}
        \Loss(\bw_t) = \frac{1}{n} \sum_{i=1}^n \loss(\bw_t^\top \bx_i)
        \geq \loss\l(c_2 \cdot t\r)
        \simeq \exp\l(-c_2 \cdot t\r)
    \end{equation*}
    The same argument applies to $\Loss(\barw_t).$
    This completes the proof.
\end{proof}

\section{Missing Proofs for Theorems in Section \ref{sec.lower.burn.in}}\label{appendix.proof.lower.bound.burnin}
\subsection{Proof of Theorem \ref{thm.lower.burn.in.maintext}}\label{appendix.proof.lower.bound.batch}
\begin{proof}[Proof of \Cref{thm.lower.burn.in.maintext}]

For $0 < \gamma < 1/6$ and $n > 16,$ we define $d := \lfloor 1/5\gamma^2 \rfloor \geq 6,$ where $\lfloor \cdot \rfloor$ is the floor function. 
Let $(\be_i)_{i=1}^d$ be a set of standard basis vectors for $\R^d$.
Note that all the first-order methods defined in \Cref{def.first-order-batch} are rotational invariants. Moreover, the criterion for all data classified correctly is also rotation invariant. Therefore, without loss of generality, 
we can assume $\bw_0$ is propotional to $\be_1$, that is, 
$\bw_0\in \mathsf{Lin}\{\be_1\}$.
   
Let $k := \min\{\lfloor \log_2 n \rfloor,\, d-2\} \geq 4.$ 
We construct a hard dataset $(\bx_{i}, y_{i})_{i=1}^n$ as follows. Let $y_i=1$ for all $i\in[n]$. 
For $j=1,\dots,k$, let 
\begin{align*}
    \bx_i := \frac{2}{\sqrt{5}}\be_{j+1} - \frac{1}{\sqrt{5}} \be_{j+2} \quad \text{for}\ \ 2^{k}-2^{k-j+1}+1 \le i \le 2^{k}-2^{k-j}.
\end{align*}
Note that $j\le k\le  d-2$, thus such $\bx_i$'s are well defined.
Let the remaining $\bx_i$'s be
\begin{align*}
   \bx_i:= \frac{1}{\sqrt{5}} \be_{k+2}\quad \text{for}\ \ 2^k \le i\le n.
\end{align*}
Note that $\|\bx_i\|\le 1$ for $i=1,\dots,n$.
Moreover, for the unit vector
\[
\wStar := \frac{1}{\sqrt{d}}  \l(1,1,\dots,1\r)^{\top}, 
\]
we have 
\[y_i \bx_i^\top\wStar = \frac{1}{\sqrt{5d}} \geq \gamma,\quad \text{for}\ i=1\dots,n.\]
Thus $(\bx_i, y_i)_{i=1}^n$ satisfies \Cref{assumption.data.maintext}.

For a vector $\bw \in \R^d$, we also write it as $\bw := (w^{(1)}, w^{(2)},\dots,w^{(d)})^\top$.
Then, the objective function can be written as
\begin{align*}
    \Loss\l(\bw\r) 
    = \frac{1}{n}\sum_{i=1}^n \loss\big(\bw^\top \bx_i\big)
    = \frac{1}{n} \Bigg[\sum_{j=1}^k 2^{k-j}   \loss\bigg(\frac{2}{\sqrt{5}} w^{(j+1)} - \frac{1}{\sqrt{5}} w^{(j+2)}\bigg) + \big(n-2^{k}+1\big)  \loss\bigg(\frac{1}{\sqrt{5}} w^{(k+2)}\bigg)\Bigg].
\end{align*}
Thus the $j$-th coordinate of $\nabla \Loss(\bw)$ is given by
\begin{align*}
    \big[\nabla \Loss(\bw)\big]_j = 
    \begin{dcases}
    0 & j=1,\\
        \frac{2^k}{\sqrt{5}n}   \lossPrime \bigg(\frac{2}{\sqrt{5}} w^{(2)} - \frac{1}{\sqrt{5}} w^{(3)}\bigg) & j = 2,\\
        \frac{2^{k-j+2}}{\sqrt{5}n} \bigg[\lossPrime\bigg(\frac{2}{\sqrt{5}} w^{(j)} - \frac{1}{\sqrt{5}} w^{(j+1)}\bigg) - \lossPrime\bigg(\frac{2}{\sqrt{5}} w^{(j-1)} - \frac{1}{\sqrt{5}}w^{(j)}\bigg)\bigg] & 3 \leq j \leq k+1,\\
        \frac{1}{\sqrt{5}n} \bigg[ \big(n-2^{k}+1\big) \lossPrime\bigg(\frac{1}{\sqrt{5}}w^{(k+2)}\bigg)-\lossPrime\bigg(\frac{2}{\sqrt{5}} w^{(k+1)} -\frac{1}{\sqrt{5}}w^{(k+2)}\bigg)\bigg] & j = k+2, \\
        0 &  \text{otherwise}.
    \end{dcases}
\end{align*}
Consider a sequence $(\bw_s)_{s=0}^t$ generated by a first-order method according to \Cref{def.first-order-batch}. 
Since $\bw_0\in \mathsf{Lin}\{\be_1\}$, the gradient at $\bw_0$ vanishes in all coordinates except the second and the $(k+2)$-th coordinates. 
Therefore we have 
\[
\bw_1 \in \mathsf{Lin}\{\be_1,\, \be_2,\, \be_{k+2}\}.
\]
Then the gradient at $\bw_1$ vanishes in all coordinates except the second, third, $(k+1)$-th, and $(k+2)$-th coordinates.
Therefore we have 
\[
\bw_2 \in \mathsf{Lin}\{\be_1,\, \be_2,\, \be_3,\, \be_{k-1},\, \be_{k+2}\}.
\]
By induction, we conclude that for $t \leq t_0-2$ for $t_0:= \lfloor (k+1)/2\rfloor,$ it holds that
\begin{equation*}
    \bw_t \in \mathsf{Lin}\{\be_1, \dots, \be_{t+1}, \be_{k+3-t}, \dots, \be_{k+2}\}.
\end{equation*}
So for all $(\bw_s)_{s=0}^{t}$, their $t_0$-th and $(t_0+1)$-th coordinates must be zero. 
By our dataset construction, there exists $i\le 2^k-1$ such that
\begin{equation*}
   y_{i} \bx_{i}^\top \bw_k= 0,\quad \text{for all}\ k=0,\dots,t.
\end{equation*}
This means that the dataset cannot be separated by any of $(\bw_k)_{k=0}^{t}$.
Thus, for the first-order method to output a linear separator, we must have 
\begin{align*}
    t &\geq t_0-1 = \l\lfloor\frac{k-1}{2}\r\rfloor 
    \geq \frac{k}{2} -1 
    \geq \min\bigg\{\frac{\log_2 n-3}{2}, \frac{d}{2}-2\bigg\}
    \geq \min\bigg\{\frac{\log_2 n}{8}, \frac{d}{3}\bigg\} \\
    &\geq \min\bigg\{\frac{\log_2 n}{8}, \frac{1}{15 \gamma^2} -\frac{1}{3}\bigg\} 
    \geq \min\bigg\{\frac{\ln n}{8 \ln 2}, \frac{1}{30 \gamma^2}\bigg\}.
\end{align*}
This completes the proof.
\end{proof}

\subsection{Proof of Theorem \ref{thm.lower.bound.one.pass.maintext}}\label{appendix.lower.bound.one.pass}
\begin{proof}[Proof of~\Cref{thm.lower.bound.one.pass.maintext}]
For $0 < \gamma < 1/2$ and $n \geq 2,$ we define $d := \lfloor 1/\gamma^2 \rfloor \geq 4,$ where $\lfloor \cdot \rfloor$ is the floor function. 
Let $(\be_i)_{i=1}^d$ be a set of standard basis vectors for $\R^d$.
Note that all the first-order online methods defined in \Cref{def-online-first-order} are rotational invariants. 
Moreover, the criterion for all data classified correctly is also rotation invariant. 
Therefore, without loss of generality, 
we can assume $\bw_0$ is propotional to $\be_1$, that is, 
$\bw_0\in \mathsf{Lin}\{\be_1\}$.

Let $k := \min\{n, d-1\} \geq 2.$
We construct a hard dataset $(\bx_i,y_i)_{i=1}^n$ as follows.
Let $y_i = 1$ for all $i \in [n]$.
For i = 1,2,...,k, let
\[
\bx_i = \be_{i+1}.
\]
Note that $k \leq d-1,$ thus such $\bx_i$'s are well defined.
Let the remaining $\bx_i$'s be
\[
\bx_i = \be_{k+1} \ \ \text{for} \ \ k+1 \leq i \leq n \ \ \text{if} \ \ n \geq k+1.
\]
Moreover, for the unit vector
\[
\wStar := \frac{1}{\sqrt{d}}  \l(1,1,\dots,1\r)^{\top}, 
\]
we have 
\[y_i \bx_i^\top\wStar = \frac{1}{\sqrt{d}} \geq \gamma,\quad \text{for}\ i=1\dots,n.\]
Thus $(\bx_i, y_i)_{i=1}^n$ satisfies \Cref{assumption.data.maintext}. 

Consider a sqeuence $(\bw_s)_{s= 0}^t$ generated by a first-order online method with initialization $\bw_0$ and dataset $\big( \bx_{\pi(s)}, y_{\pi(s)} \big)_{s=1}^t$, where $\pi(s)\in[n]$ but is otherwise arbitrary.
Since $\bw_0\in \mathsf{Lin}\{\be_1\}$, the gradient at $\bw_0$ vanishes in all coordinates except the direction of $\lossPrime(y_{\pi(1)} \bx_{\pi(1)}^\top \bw_0) y_{\pi(1)} \bx_{\pi(1)}$. 
Therefore we have 
\[
\bw_1 \in \mathsf{Lin}\{\be_1,\, \bx_{\pi(1)}\}.
\]
Then the gradient at $\bw_1$ vanishes in all coordinates except the span of $\lossPrime(y_{\pi(1)} \bx_{\pi(1)}^\top \bw_0) y_{\pi(1)} \bx_{\pi(1)}$ and $\lossPrime(y_{\pi(2)} \bx_{\pi(2)}^\top \bw_0) y_{\pi(2)} \bx_{\pi(2)}$.
Therefore we have 
\[
\bw_2 \in \mathsf{Lin}\{\be_1,\, \bx_{\pi(1)}, \bx_{\pi(2)}\}.
\]
By induction, we conclude that for all $t \leq k-1$, it holds that
\begin{equation*}
    \bw_t \in \mathsf{Lin}\{\be_1, \bx_{\pi(1)}, \dots, \bx_{\pi(t)}\}.
\end{equation*}
From our dataset construction, there exists $j \in [d]$, such that for all $(\bw_s)_{s=0}^{t}$, their $j$-th coordinates are zero. 
Therefore,
\begin{equation*}
   y_{j} \bx_{j}^\top \bw_s= 0,\quad \text{for all}\ s=0,\dots,t.
\end{equation*}
This means that the dataset cannot be separated by any of $(\bw_s)_{s=0}^{t}$.
Thus, for the first-order online method to output a linear separator, we must have 
\begin{align*}
    t &\geq k = \min\{n,d-1\} \geq \min\bigg\{\frac{1}{2\gamma^2},n\bigg\}. 
\end{align*}
This proves the first claim.

For the second claim, consider a sqeuence $(\bw_s)_{s= 0}^t$ generated by a first-order online method with initialization $\bw_0$ and dataset $(\bx_{i}, y_{i})_{i=1}^n$ as defiend previously.
By the same induction, we conclude that for all $0 \leq t \leq k-1,$ it holds that
\begin{equation*}
    \bw_t \in \mathsf{Lin} \{\be_1, \be_2, \dots, \be_{t+1}\}.
\end{equation*}
Since $\bx_{t+1} = \be_{t+2}$ for $0 \leq t \leq k-1,$ this implies
\begin{equation*}
    y_{t+1} \bx_{t+1}^\top \bw_{t} = 0,
\end{equation*}
which suggests that the algorithm makes a mistake step.
Therefore, the total mistake steps up to the $t$-th step is
\begin{equation*}
    \sum_{j=0}^{t-1} \mathbbm{1}\{y_{j+1} \bx_{j+1}^\top  \bw_{j}\le 0\}
    \geq \min \{t, k\} = \min\{t, d-1\}
    \geq \min\bigg\{\frac{1}{2\gamma^2}, t\bigg\}.
\end{equation*}
This completes the proof.
\end{proof}

\subsection{An alternative lower bound for first-order batch algorithms}\label{appendix.alternate.lower.bound}
Now we present another lower bound for first-order batch algorithms with a different trade-off than~\Cref{thm.lower.burn.in.maintext}.
We leave it as an open problem to figure out the optimal trade-off between the sample size $n$ and margin $\gamma$ in finding linear separators.

\begin{theorem}[An alternate lower bound for batch first-order methods]\label{thm.lower.burn.in.alternate}
For every $0 < \gamma < 1/8$, $n\geq 4$, and $\bw_0$,
there exists a dataset $\l(\bx_i,y_i\r)_{i=1}^n$ 
satisfying Assumption~\ref{assumption.data.maintext} 
such that the following holds.
For any $\bw_t$ output 
by a first-order batch method in $t$-steps with initialization $\bw_0$ on this dataset, we have
\begin{equation*}
\min_{i\in[n]} y_i \bx_i^\top  \bw_t> 0
\quad \text{implies that}\quad 
    t \ge \min\bigg\{\frac{n}{4}, \frac{1}{8\gamma^{2/3}}\bigg\}.
\end{equation*}
\end{theorem}
\begin{proof}[Proof of~\Cref{thm.lower.burn.in.alternate}]

For $0 < \gamma < 1/8$ and $n \geq 2,$ we define $d := \lfloor \gamma^{-2/3}\rfloor \geq 4,$ where $\lfloor \cdot \rfloor$ is the floor function. 
Let $(\be_i)_{i=1}^d$ be a set of standard basis vectors for $\R^d$.
Note that all the first-order methods defined in \Cref{def.first-order-batch} are rotational invariants. 
Moreover, the criterion for all data classified correctly is also rotation invariant. Therefore, without loss of generality, 
we can assume $\bw_0$ is propotional to $\be_1$, that is, 
$\bw_0\in \mathsf{Lin}\{\be_1\}$.
   
Let $k := \min\{n,\, d-2\}.$ 
We construct a hard dataset $(\bx_{i}, y_{i})_{i=1}^n$ as follows. Let $y_i=1$ for all $i\in[n]$. 
For $j=1,\dots,k$, let 
\begin{align*}
    \bx_i := -\frac{1}{\sqrt{2}}\be_{j+1} + \frac{1}{\sqrt{2}} \be_{j+2}.
\end{align*}
Let the remaining $\bx_i$'s be
\begin{align*}
   \bx_i:= \frac{1}{\sqrt{2}} \be_{k+2}\quad \text{for}\ \ k+1 \le i\le n \ \ \text{if} \ \ n \geq k+1.
\end{align*}
Note that $j\le k\le  d-2$, thus such $\bx_i$'s are well defined.
Also note that $\|\bx_i\|\le 1$ for $i=1,\dots,n$.
Moreover, for the unit vector
\[
\wStar := \sqrt{\frac{6}{d(d+1)(2d+1)}}  \l(1,2,\dots,d\r)^{\top}, 
\]
we have 
\[
y_i \bx_i^\top\wStar = \sqrt{\frac{3}{d(d+1)(2d+1)}} \geq 
\sqrt{\frac{1}{d^3}} \geq
\gamma,\quad \text{for}\ \ i=1\dots,n.
\]
Thus $(\bx_i, y_i)_{i=1}^n$ satisfies \Cref{assumption.data.maintext}.

For a vector $\bw \in \R^d$, we also write it as $\bw := (w^{(1)}, w^{(2)},\dots,w^{(d)})^\top$.
Then, the objective function can be written as
\begin{align*}
    \Loss\l(\bw\r) 
    = \frac{1}{n}\sum_{i=1}^n \loss\big(\bw^\top \bx_i\big)
    = \frac{1}{n} \Bigg[\sum_{j=1}^k \loss\bigg(-\frac{1}{\sqrt{2}} w^{(j+1)} + \frac{1}{\sqrt{2}} w^{(j+2)}\bigg) + \big(n-k\big)  \loss\bigg(\frac{1}{\sqrt{2}} w^{(k+2)}\bigg)\Bigg].
\end{align*}
Thus the $j$-th coordinate of $\nabla \Loss(\bw)$ is given by
\begin{align*}
    \big[\nabla \Loss(\bw)\big]_j = 
    \begin{dcases}
    0 & j=1,\\
        -\frac{1}{\sqrt{2}n}   \lossPrime \bigg(-\frac{1}{\sqrt{2}} w^{(2)} + \frac{1}{\sqrt{2}} w^{(3)}\bigg) & j = 2,\\
        \frac{1}{\sqrt{2}n} \bigg[\lossPrime\bigg(-\frac{1}{\sqrt{2}} w^{(j-1)} + \frac{1}{\sqrt{2}} w^{(j)}\bigg) - \lossPrime\bigg(-\frac{1}{\sqrt{2}} w^{(j)} + \frac{1}{\sqrt{2}}w^{(j+1)}\bigg)\bigg] & 3 \leq j \leq k+1,\\
        \frac{1}{\sqrt{2}n} \bigg[ \big(n-k\big) \lossPrime\bigg(\frac{1}{\sqrt{2}}w^{(k+2)}\bigg)-\lossPrime\bigg(-\frac{1}{\sqrt{2}} w^{(k+1)} +\frac{1}{\sqrt{2}}w^{(k+2)}\bigg)\bigg] & j = k+2, \\
        0 &  \text{otherwise}.
    \end{dcases}
\end{align*}
Consider a sequence $(\bw_s)_{s=0}^t$ generated by a first-order method according to \Cref{def.first-order-batch}. 
Since $\bw_0\in \mathsf{Lin}\{\be_1\}$, the gradient at $\bw_0$ vanishes in all coordinates except the second and the $(k+2)$-th coordinates. 
Therefore we have 
\[
\bw_1 \in \mathsf{Lin}\{\be_1,\, \be_2,\, \be_{k+2}\}.
\]
Then the gradient at $\bw_1$ vanishes in all coordinates except the second, third, $(k+1)$-th, and $(k+2)$-th coordinates.
Therefore we have 
\[
\bw_2 \in \mathsf{Lin}\{\be_1,\, \be_2,\, \be_3,\, \be_{k-1},\, \be_{k+2}\}.
\]
By induction, we conclude that for $t \leq t_0-2$ with $t_0:= \lfloor (k+1)/2\rfloor,$ it holds that
\begin{equation*}
    \bw_t \in \mathsf{Lin}\{\be_1, \dots, \be_{t+1}, \be_{k+3-t}, \dots, \be_{k+2}\}.
\end{equation*}
So for all $(\bw_s)_{s=0}^{t}$, their $t_0$-th and $(t_0+1)$-th coordinates must be zero. 
By our dataset construction, there exists $i\le 2^k-1$ such that
\begin{equation*}
   y_{i} \bx_{i}^\top \bw_k= 0,\quad \text{for all}\ k=0,\dots,t.
\end{equation*}
This means that the dataset cannot be separated by any of $(\bw_k)_{k=0}^{t}$.
Thus, for the first-order method to output a linear separator, we must have 
\begin{align*}
    t &\geq t_0-1 = \l\lfloor\frac{k-1}{2}\r\rfloor 
    \geq \frac{k}{2} -1 
    \geq \min\bigg\{\frac{n}{2}-1, \frac{d}{2}-2\bigg\} 
    \geq \min\bigg\{\frac{n}{4}, \frac{d}{4}\bigg\} \\
    &\geq \min\bigg\{\frac{n}{4}, \frac{1}{4}\big(\gamma^{-2/3}-1\big)\bigg\}
    \geq \min\bigg\{\frac{n}{4}, \frac{1}{8\gamma^{2/3}}\bigg\}
\end{align*}
This completes the proof.
\end{proof}

\section{A Generalization of Theorem~\ref{thm.NN.maintext} and Its Proof}\label{appendix.proof.NN} 
In this section, we generalize~\Cref{thm.NN.maintext} to more general activation functions and present its proof.
First, we provide general assumptions on the activation functions.
\begin{assumption}[Activation function]\label{assumption.activation.functions.maintext}
    In two-layer neural networks \eqref{eqn.def.NN.maintext}, let the activation function $\sigma(z): \R \to \R$ be a locally Lipschitz function.
    For each $z,$ let $\sigma^{\prime}(z)$ be a unique element from the Clarke subdifferential of $\sigma(\cdot)$ at $z$ \citep{clarke1990optimization}.
    Moreover, we assume
    \begin{itemize}
        \item 
        There exists $0 < \alpha < 1$ such that $\sigma^\prime(z)\in [\alpha,1].$
        \item For any $z \in \R,$ it holds that $\l|\sigma(z) - \sigma^{\prime} z \r| \leq \kappa$ for some $\kappa > 0.$ 
    \end{itemize}
\end{assumption}

\paragraph{Examples of activation functions.}
Assumption \ref{assumption.activation.functions.maintext} holds for a broad class of leaky activations. 
For instance, let $\sigma_*(\cdot)$ be one of the following: GeLU $\sigma_{\GeLU}(x) :=  x \cdot F(x),$ where $F(x)$ is the cumulative density function of the standard Gaussian distribution; Softplus $\sigma_{\Softplus}(x) := \ln(1 + \exp(x));$ SiLU $\sigma_{\SiLU}(x) := x/(1 + \exp(-x));$ ReLU $\sigma_{\ReLU}(x) := \max\{x,0\}.$
Then, we fix some constant $0.5 < c <1$ and define
$
\sigma(x) = c \cdot x + \frac{1-c}{4} \cdot \sigma_*(x),
$
where $\sigma_*$ can be any activation function above. 
It is straightforward to check that such a ``leaky'' variant satisfies~\Cref{assumption.activation.functions.maintext} (see Example 2.1 in~\cite{cai2024large}).
Moreover, the leaky ReLU activation defined as $\sigma(z) = \max\{z,\alpha z\}$ satisfies~\Cref{assumption.activation.functions.maintext} with $\alpha$ and $\kappa = 0.$

Now, we present the improved convergence rate for two-layer neural networks. Setting $\kappa = 0$ for leaky ReLU function recovers~\Cref{thm.NN.maintext}.
\begin{theorem}[General two-layer neural networks]\label{thm.NN.general}
Consider~\Cref{eqn.GD.logistic.maintext} on~\Cref{eqn.loss.NN.maintext} with adaptive stepsizes~\Cref{eqn.lr.scheduler.maintext} for two-layer neural networks~\Cref{eqn.def.NN.maintext} under~\Cref{assumption.data.maintext} and~\Cref{assumption.activation.functions.maintext}.
Assume without loss of generality that $\bw_0 = \mathbf{0}.$
Then for \emph{every} $t \geq 1$ and $\lr > 0$, we have
\begin{equation*}
    \min_{k\leq t} \Loss(\bw_k) \leq 
        \exp\l(\kappa - \frac{(\alpha \gamma^2 (t+1))^2 - 1}{4 \gamma^2 (t+1)} \lr\r).
\end{equation*}
In particular, after $1/(\alpha \gamma^2)$ burn-in steps, for every $\lr > 0,$ we have
\begin{equation}
\min_{k\leq t} \Loss(\bw_k) \leq \exp\l(\kappa -\frac{\alpha^2 \gamma^2 \lr}{4}\r)
= \exp\l(-\Theta\l(\lr\r)\r), \quad t \geq \frac{1}{\alpha \gamma^2}.
\end{equation}
\end{theorem}

Before we prove~\Cref{thm.NN.general}, we first state and prove two lemmas.
\begin{lemma}\label{lem.NN.1.maintext}
    Under~\Cref{assumption.data.maintext}, for every $\bw \in \R^d$, it holds that
    \begin{equation*}
        I_2(\bw) := 2 \innerProduct{m \cdot \nabla \phi(\bw)}{\bu_2} + \lr \norm{}{m \cdot \nabla\phi(\bw)}^2 \leq 0 \ \ \text{for} \ \ \bu_2^{(j)} := \frac{a_j \lr}{2 \gamma} \wStar, \ j = 1,2,...m.
    \end{equation*}
\end{lemma}

\begin{proof}[Proof of \Cref{lem.NN.1.maintext}]
    Define
    $
        g_{i,j} := \lossPrime(y_i f(\bw,\bx_i)) \cdot \sigma^{\prime}(\bx_i^\top \bw^{(j)}) \leq 0
    $
    for $i \in [n]$ and $j \in [m].$
    We know $|g_{i,j}| \leq |\lossPrime(y_i f(\bw,\bx_i))|$ since $\sigma^{\prime}(\cdot) \leq 1.$
    Then, for $j \in [m]$, we have
    \begin{align*}
        \norm{}{m \cdot \nabla_{\bw^{(j)}} \phi(\bw)}
        = \norm{}{\l(-\loss^{-1}\r)^{\prime}\l(\Loss(\bw)\r) \cdot \frac{1}{n} \sum_{i=1}^n a_j g_{i,j} y_i \bx_i} 
        \leq \l|\l(-\loss^{-1}\r)^{\prime}\l(\Loss(\bw)\r)\r| \cdot \frac{1}{n} \sum_{i=1}^n |\lossPrime(y_i f(\bw,\bx_i))|.
    \end{align*}
    Let $\loss_i := \loss(y_i f(\bw, \bx_i)).$
    We arrange the term above as
    \begin{align*}
        \l|\l(-\loss^{-1}\r)^{\prime}\l(\Loss(\bw)\r)\r| \cdot \frac{1}{n} \sum_{i=1}^n |\lossPrime(y_i f(\bw,\bx_i))|
        = \frac{\frac{1}{n} \sum_{i=1}^n \l(-\lossPrime\r) \l(\loss^{-1}\l(\loss_i\r)\r)}{\l(-\lossPrime\r) \l(\loss^{-1} \l(\frac{1}{n} \sum_{i=1}^n \loss_i\r)\r)}
        = \frac{\frac{1}{n} \sum_{=1}^n h\l(\loss_i\r)}{h\l(\frac{1}{n} \sum_{i=1}^n \loss_i\r)},
    \end{align*}
    where $h(\cdot)$ is defined as $h(z) := \l(-\lossPrime\r) \l(\loss^{-1}\l(z\r) \r).$
    For both losses, $h(\cdot)$ is concave on $z > 0$ (see~\Cref{lemma.1.lip} in~\Cref{appendix.logistic.result}). Therefore, we have $\norm{}{m \nabla_{\bw^{(j)}} \phi(\bw)}^2 \leq 1$ for $j \in [m].$ Apply this upper bound of $\norm{}{\nabla \phi(\bw)}$, and recall~\Cref{assumption.data.maintext}, we have
    \begin{align*}
        I_2(\bw) &= \frac{2}{n} \l(-\loss^{-1}\r)^{\prime}\l(\Loss(\bw)\r) \cdot \sum_{j=1}^m \sum_{i=1}^n g_{i,j} \cdot \norm{}{\bu_2^{(j)}} \cdot y_i \bx_i^\top \wStar + \lr \cdot \sum_{j=1}^m \norm{}{m \nabla_{\bw^{(j)}} \phi\l(\bw\r)}^2 \\
        &\leq \frac{2}{n} \l(-\loss^{-1}\r)^{\prime}\l(\Loss(\bw)\r) \cdot \sum_{j=1}^m \sum_{i=1}^n g_{i,j} \cdot \norm{}{\bu_2^{(j)}} \cdot y_i \bx_i^\top \wStar 
        + \lr \cdot \sum_{j=1}^m \norm{}{m \nabla_{\bw^{(j)}} \phi\l(\bw\r)} \\
        &\leq \frac{-2\gamma}{n} \l(-\loss^{-1}\r)^{\prime}\l(\Loss(\bw)\r) \cdot \sum_{j=1}^m \sum_{i=1}^n \l|g_{i,j}\r| \cdot \norm{}{\bu_2^{(j)}} 
        + \lr \l(-\loss^{-1}\r)^{\prime}\l(\Loss(\bw)\r) \cdot \frac{1}{n} \sum_{j=1}^m \sum_{i=1}^n \l|g_{i,j}\r| \\
        &= \l(-\loss^{-1}\r)^{\prime}\l(\Loss(\bw)\r) \cdot \frac{1}{n} \sum_{j=1}^m \sum_{i=1}^n \l|g_{i,j}\r| \cdot \l(-2\gamma \norm{}{\bu_2^{(j)}} + \lr\r).
    \end{align*}
    Invoking the definition of $\bu_2,$ we complete the proof.
\end{proof}

\begin{lemma}\label{lem.NN.2.maintext}
    Under~\Cref{assumption.data.maintext}, if $\bu_1^{(j)} \propto a_j \cdot \wStar$ for $j=1,2,...,m,$ then for any $\bw \in \R^d,$
    \begin{equation*}
        I_1(\bw) := \innerProduct{\nabla \phi(\bw)}{\bu_1 - \bw} \leq \kappa - \frac{\alpha \gamma}{m} \sum_{j=1}^m \norm{}{\bu_1^{(j)}} - \phi\l(\bw\r).
    \end{equation*}
\end{lemma}

\begin{proof}[Proof of \Cref{lem.NN.2.maintext}]
    From the definition of the transformed loss function $\phi(\cdot),$ we have
    \begin{align*}
        I_1(\bw) 
        &= \l(-\loss^{-1}\r)^{\prime} \l(\Loss(\bw)\r) \cdot \frac{1}{n} \sum_{i=1}^n \lossPrime\l(y_i f\l(\bw;\bx_i\r)\r) \cdot \frac{1}{m}\sum_{j=1}^m y_i a_j \sigma^{\prime} \l(\bx_i^\top \bw^{(j)}\r) \cdot \bx_i^\top \l(\bu_1^{(j)} - \bw^{(j)}\r) \\
        &= \l(-\loss^{-1}\r)^{\prime} \l(\Loss(\bw)\r) \cdot \frac{1}{n} \sum_{i=1}^n \lossPrime\l(y_i f\l(\bw;\bx_i\r)\r) \cdot \l[J_{i} - y_i f\l(\bw,\bx_i\r)\r],
    \end{align*}
    where 
    \begin{align*}
        J_{i} 
        &:= \frac{1}{m}\sum_{j=1}^m a_j \l[
        \sigma^{\prime}\l(\bx_i^\top \bw^{(j)}\r) y_i \bx_i^\top \bu_1^{(j)} \r]
        + \frac{1}{m}\sum_{j=1}^m y_i a_j \underbrace{\l[ \sigma\l(\bx_i^\top \bw^{(j)}\r)
        - \sigma^{\prime} \l(\bx_i^\top \bw^{(j)}\r) \bx_i^\top \bw^{(j)}\r]}_{\l|\cdot\r| \leq \kappa} \\
        &\geq \frac{1}{m}\sum_{j=1}^m a_j \l[
        \sigma^{\prime}\l(\bx_i^\top \bw^{(j)}\r) y_i \bx_i^\top \bu_1^{(j)} \r] - \kappa
        \geq \frac{\alpha \gamma}{m} \sum_{j=1}^m \norm{}{\bu_1^{(j)}} - \kappa,
    \end{align*}
    where the last inequality utilizes $\sigma^{\prime}(\cdot) \geq \alpha, \bu^{(j)}_1 \propto a_j \wStar,$ and~\Cref{assumption.data.maintext}.
    Notice $\lossPrime(\cdot) \leq 0,$ we define $\bz := \l(y_1 f\l(\bw;\bx_1\r), y_2 f\l(\bw,\bx_2\r),..., y_n f\l(\bw,\bx_n\r)\r)^\top$ and $\bones_n = \l(1,1,...,1\r)^\top \in \R^n.$ 
    We also define
    $\psi(\bz) = (-\loss^{-1})(1/n \cdot \sum_{i=1}^n \loss(\bz_i)).$ 
    From the definition, we know $\phi(\bw) = \psi(\bz)$ and $\psi(\cdot)$ is convex for both exponential and logistic loss (Theorem 5.1 in \citep{ji2021characterizing}, also see~\Cref{lem.convex.phi}). Therefore, 
    \begin{align*}
         I_1(\bw) 
         &\leq \l(-\loss^{-1}\r)^{\prime} \l(\Loss(\bw)\r) \cdot \frac{1}{n} \sum_{i=1}^n \lossPrime\l(y_i f\l(\bw;\bx_i\r)\r) \cdot \l(\frac{\alpha \gamma}{m} \sum_{j=1}^m \norm{}{\bu_1^{(j)}} -\kappa - y_i f\l(\bw,\bx_i\r)\r) \\
         &= \innerProduct{\nabla \psi \l(\bz\r)}{\l(\frac{\alpha \gamma}{m} \sum_{j=1}^m \norm{}{\bu_1^{(j)}} -\kappa\r)\cdot \bones_n - \bz}
         \leq \psi\l(\l(\frac{\alpha \gamma}{m} \sum_{j=1}^m \norm{}{\bu_1^{(j)}} -\kappa\r)\cdot \bones_n\r) - \psi\l(\bz\r) \\
         &= \kappa - \frac{\alpha \gamma}{m} \sum_{j=1}^m \norm{}{\bu_1^{(j)}} - \phi\l(\bw\r).
    \end{align*}
    This completes the proof.%\LL{Should the second $\leq$ be  $=$?} \RZ{Yes.}
\end{proof}

Now we present the proof of~\Cref{thm.NN.general}
\begin{proof}[Proof of~\Cref{thm.NN.general}]
    As explained in \Cref{eqn.transformed.loss.maintext}, it is equivalent to considering GD with a constant stepsize under a transformed objective $\phi(\cdot)$. 
    We then use the split optimization technique developed by \citet{wu2024large} and \citet{cai2024large}.
    Specifically, consider a comparator $\bu = \bu_1 + \bu_2 \in \R^{dm}$:
        \begin{equation*}
            \bu_1 = \begin{pmatrix}
                \bu_1^{(1)} \\ \bu_1^{(2)} \\ ... \\ \bu_1^{(m)}
            \end{pmatrix}, \quad 
            \bu_2 = \begin{pmatrix}
                \bu_2^{(1)} \\ \bu_2^{(2)} \\ ... \\ \bu_2^{(m)}
            \end{pmatrix}.
        \end{equation*}
    From GD iterate in~\Cref{eqn.transformed.loss.maintext}, we have
    \begin{align}\label{eqn.NN.split.1}
        \norm{}{\bw_{t+1} - \bu}^2
        =~& \norm{}{\bw_t - \bu}^2 
        + 2 \lr m  \innerProduct{\nabla \phi(\bw_t)}{\bu_1 - \bw_t} + \lr \l(2 \innerProduct{m \cdot \nabla \phi(\bw_t)}{\bu_2} + \lr \norm{}{m \cdot \nabla\phi(\bw_t)}^2\r) \notag \\
        \leq~& \norm{}{\bw_t - \bu}^2 + 2 \lr m \bigg(\kappa - \frac{\alpha \gamma}{m} \sum_{j=1}^m \norm{}{\bu_1^{(j)}} - \phi\l(\bw_t\r)\bigg),
    \end{align}
    where~\Cref{eqn.NN.split.1} is by the following two lemmas.
    Rearraging \eqref{eqn.NN.split.1} and telescoping the sum to obtain
    \begin{align}
        &\frac{\norm{}{\bw_t - \bu}^2}{2 \lr m (t+1)} + \frac{1}{t+1} \sum_{k=0}^{t} \phi(\bw_k) \leq \kappa - \frac{\alpha \gamma}{m} \sum_{j=1}^m \norm{}{\bu_1^{(j)}} + \frac{\norm{}{\bu}^2}{2 \lr m (t+1)}.
    \end{align}
    Further setting $\|\bu_1^{(j)}\| = \alpha \gamma \lr (t+1)/2,$ we get
    \begin{align}
        &\frac{1}{t+1} \sum_{k=0}^{t} \phi(\bw_k) \leq \kappa - \frac{\alpha \gamma}{m} \sum_{j=1}^m \norm{}{\bu_1^{(j)}} + \frac{\norm{}{\bu_1 + \bu_2}^2}{2 \lr m (t+1)}
        \leq \kappa - \frac{(\alpha \gamma^2 (t+1))^2 - 1}{4 \gamma^2 (t+1)} \lr.
    \end{align}
    We complete the proof by applying the fact that $\Loss(\cdot) = \loss\l(-\phi(\cdot)\r).$
\end{proof}

\section{Proofs for Theorem~\ref{thm.general.loss.maintext} and Example~\ref{lem.example}}\label{appendix.proof.general.loss}

\begin{proof}[Proof of \Cref{thm.general.loss.maintext}]
    The entire proof is analog to the proof of~\Cref{thm.LR.maintext} in~\Cref{sec.proof.sketch.maintext}. 
    Define $\bu = \bu_1 + \bu_2,$ where $\bu_2 = (C_\loss \lr/2 \gamma) \cdot \wStar \in \R^d,$ where $C_\loss$ is the loss-dependent constant in~\Cref{assumptions.general.loss.maintext}.
    Recall the gradient descent iterate \eqref{eqn.GD.logistic.maintext}, the learning rate scheduler \eqref{eqn.lr.scheduler.maintext}, and the definition for $\phi(\cdot).$ Then,
    \begin{align}\label{eqn.proof.general.1}
        \norm{}{\bw_{t+1} - \bu}^2
        &~= \norm{}{\bw_{t} - \bu}^2 + 2\lr \innerProduct{\nabla \phi(\bw_t)}{\bu - \bw_t} + \lr^2 \norm{}{\nabla
        \phi(\bw_t)}^2 \notag \\
        &~= \norm{}{\bw_{t} - \bu}^2 + 2\lr \innerProduct{\nabla \phi(\bw_t)}{\bu_1 - \bw_t} + \lr \l[2 \innerProduct{\nabla \phi(\bw_t)}{\bu_2} + \lr \norm{}{\nabla \phi(\bw_t)}^2 \r] \notag \\
        &~\overset{(a)}{\leq} \norm{}{\bw_{t} - \bu}^2 + 2\lr \innerProduct{\nabla \phi(\bw_t)}{\bu_1 - \bw_t}
        \overset{(b)}{\leq} \norm{}{\bw_{t} - \bu}^2 + 2\lr \l(\phi(\bu_1) - \phi(\bw_t)\r).
    \end{align}
    Here, $(a)$ is from~\Cref{lem.key.maintext} and $(b)$ is from the convexity of $\phi(\cdot)$ following~\Cref{asp.item.B}.
    Rearranging the inequality above and telescoping the sum, we have
    \begin{equation}\label{eqn.split.optimization.bound.general.loss}
        \frac{\norm{}{\bw_t - \bu}^2}{2\lr (t+1)} + \frac{1}{t+1} \sum_{k=0}^{t} \phi(\bw_k)
        \leq \phi(\bu_1)
        + \frac{\norm{}{\bu}^2}{2\lr (t+1)}.
    \end{equation}  
    Recall $\innerProduct{\wStar}{\bx_i} \geq \gamma,$ we have $\phi(\bu_1) \leq -\gamma \norm{}{\bu_1}.$
    Then, setting $\bu_1 = (\gamma \lr (t+1)/2) \cdot \bw^*$ and invoking $\bw_0 = \mathbf{0}$ gives
    \begin{align*}
        &\frac{1}{t+1} \sum_{k=0}^{t} \phi(\bw_k)
        \leq
        -\gamma \norm{}{\bu_1}
        + \frac{\norm{2}{\bu_1 + \bu_2}^2}{2\lr (t+1)} 
        \leq -\frac{\l(\gamma^2 (t+1)\r)^2 - C_\loss}{4 \gamma^2 (t+1)} \lr.
    \end{align*}
    Finally, we complete the proof by recalling the convexity of $\phi(\cdot)$ and $\Loss(\cdot) = \loss(-\phi(\cdot)).$
\end{proof}

\begin{proof}[Proof of \Cref{lem.example}]
For exponential loss and logistic loss,~\Cref{asp.item.A} and~\Cref{asp.item.B} are trivial, and~\Cref{asp.item.C} is proven by~\Cref{lemma.1.lip} and~\Cref{lem.convex.phi}.

For the polynomial loss $\loss_{\mathsf{poly}}$\citep{ji2021characterizing}, it is straightforward to verify~\Cref{asp.item.A} and~\Cref{asp.item.B} by taking first and second order derivatives. 
We know $\loss(\cdot)$ is twice continuously differentiable, and its inverse function $\loss^{-1}(z) = z^{-1/k} - 1 \ (z > 0)$ is continuously differentiable. 
To verify the convexity in~\Cref{asp.item.C}, we can use~\Cref{lem.convex.phi} and show that $\frac{\lossPrime(t)^2}{\loss(t) \lossTwoPrime(t)}$ is decreasing on $\R$ \citep[Theorem 5.1]{ji2021characterizing}.
To further verify the Lipschitzness of $\phi(\cdot)$, let's define
$h(z) := \l|\loss^{\prime}(z)\r|/(k \cdot \loss(z)^{\frac{k+1}{k}}).$
Observe that $h(\cdot)$ is continuously differentiable, $h(z) = 1$ for $z \geq 0,$ $\lim_{z \to 0^-} h(z) = 1$ and $\lim_{z \to -\infty} h(z) = 0.$ 
In addition, we can show that $h(z)$ is increasing for $z \leq 0$ by taking derivatives.
This implies $h(z) \leq 1$, and hence, $\l|\loss^{\prime}(z)\r| \leq k \cdot \loss(z)^{\frac{k+1}{k}}$, for every $z \in \R.$
Denote $z_i = y_i \bx^\top \bw$. We have 
\begin{align*}
    \norm{}{\nabla \phi(\bw)}
    &\leq \frac{1}{n} \sum_{i=1}^n \l|\loss^{\prime}(z_i)\r| \cdot \l|\l(\loss^{-1}\r)^{\prime}\l(\frac{1}{n}\sum_{i=1}^n \loss(z_i)\r)\r|
    \leq \frac{k \cdot \frac{1}{n} \sum_{i=1}^n \loss(z_i)^{\frac{k+1}{k}}}{k \cdot \l(\frac{1}{n}\sum_{i=1}^n \loss(z_i)\r)^{\frac{k+1}{k}}} \\
    &\leq \l(\frac{\max_{1\leq i \leq n} \loss(z_i)}{\frac{1}{n}\sum_{i=1}^n \loss(z_i)}\r)^{\frac{1}{k}} 
    \leq n^{1/k}.
\end{align*}
This proves $\phi(\cdot)$ satisfies~\Cref{asp.item.C} with $C_\loss = n^{1/k}.$

For the semi-circle loss~\citep{shen2005loss}, one can also verify~\Cref{asp.item.A} and~\Cref{asp.item.B} simply by taking derivatives and compute the inverse function $\loss^{-1}(z) = 1/z - z$ for $z > 0$. To verify the convexity of $\phi(\cdot),$ we use~\Cref{lem.convex.phi} and we can show the following quantity is decreaining along $\R:$
\begin{equation*}
    \frac{(\lossPrime(z))^2}{\loss(z) \lossTwoPrime(z)}
    = \frac{\sqrt{z^2+4}}{\sqrt{z^2 + 4} + z}.
\end{equation*}
Finally, to verify the Lipschitzness in~\Cref{asp.item.C}, we define $z_i = y_i \bx_i^\top \bw.$ Note that $|\lossPrime(z)| \leq 1$ for all $z,$ and $(\loss^{-1})^{\prime}(z) = 1 + 1/z^2$ for $z > 0.$ 
We then have
\begin{equation*}
    \norm{}{\nabla \phi(\bw)}
    \leq \frac{1}{n} \sum_{i=1}^n \l|\loss^{\prime}(z_i)\r| \cdot \l|\l(\loss^{-1}\r)^{\prime}\l(\frac{1}{n}\sum_{i=1}^n \loss(z_i)\r)\r|
    \leq \underbrace{\frac{1}{n} \sum_{i=1}^n \l|\loss^{\prime}(z_i)\r|}_{\leq 1} + \underbrace{\frac{\frac{1}{n} \sum_{i=1}^n \l|\loss^{\prime}(z_i)\r|}{\l(\frac{1}{n}\sum_{i=1}^n \loss(z_i)\r)^2}}_{\leq n} \leq 1+n.
\end{equation*}
The upper bound for the second addition comes from the fact that $|\lossPrime(z)|  \leq \loss^2(z)$ for any $z.$
Therefore, the semi-circle loss satisfies~\Cref{asp.item.C} with $C_\loss = n+1.$ 
    
\end{proof}

\end{document}